\documentclass[english]{article}

\usepackage[T1]{fontenc}
\usepackage[latin9]{inputenc}
\usepackage{babel}
\usepackage[nonatbib,preprint]{my2020}
\usepackage{microtype}
\usepackage{graphicx}
\usepackage{subfigure}
\usepackage{booktabs} 
\usepackage{multirow}
\usepackage{float}
\usepackage{tabularx}
\usepackage{overpic}
\usepackage[numbers]{natbib}
\usepackage{xcolor}
\definecolor{darkblue}{rgb}{0, 0, 0.5}
\usepackage[colorlinks=true,linkcolor=darkblue,citecolor=darkblue,urlcolor=darkblue]{hyperref}

\usepackage{qiangstyle}

\title{Steepest Neural Architecture Descent:
Escaping Local Optimum with Signed Neuron Splittings}

\author{%
  \large Lemeng Wu \thanks{Equal Contribution} \\
  UT Austin \\
  \texttt{lmwu@cs.utexas.edu} \\
  \And
  \large Mao Ye \textsuperscript{*} \\
  UT Austin\\
  \texttt{my21@cs.utexas.edu} \\
  \And
  \large Qi Lei \\
  Princeton University \\
  \texttt{qilei@princeton.edu} \\
  \AND
  \large Jason D. Lee \\
  Princeton University \\
  \texttt{jasonlee@princeton.edu} \\
  \And
  \large Qiang Liu \\
  UT Austin \\
  \texttt{lqiang@cs.utexas.edu} \\
}

\global\long\def\th{\boldsymbol{\theta}}%
\global\long\def\E{\mathbb{E}}%
\global\long\def\S{S} 
\global\long\def\Vec{\textbf{\textbf{Vec}}}%
\global\long\def\X{\textbf{X}}%
\global\long\def\R{\mathbb{R}}%

\begin{document}

\maketitle

\begin{abstract}
Developing efficient and principled neural architecture optimization
methods is a critical challenge of modern deep learning.
Recently, \citet{splitting2019} proposed
a \emph{splitting steepest descent (S2D)} method
that jointly  optimizes  the neural parameters and architectures  based on
 progressively growing network structures by splitting neurons into multiple copies in a steepest descent fashion.
 However, S2D suffers from a local optimality issue when all the neurons become ``splitting stable'', a concept akin to local stability in parametric optimization.
 In this work, we develop a significant and surprising extension of the
 splitting descent framework that addresses the local optimality issue.
The idea is to observe that the original S2D is unnecessarily restricted to splitting neurons into \emph{positive weighted} copies.
By simply allowing both positive and negative weights during splitting,
we can eliminate the appearance of splitting stability in S2D and
hence escape the local optima to obtain better performance.
By incorporating signed splittings,
we significantly extend the optimization power of splitting steepest descent both theoretically and empirically.
We verify our method on
various challenging benchmarks such as
CIFAR-100, ImageNet and ModelNet40, on which we outperform S2D and other advanced methods on learning accurate and energy-efficient neural networks.
\end{abstract}

\section{Introduction}
Although the parameter learning of deep neural networks (DNNs) has been well addressed by gradient-based optimization, 
efficient optimization of neural network architectures (or structures) is still largely open. Traditional approaches frame the neural architecture optimization as a discrete combinatorial optimization problem, 
which, however, often lead to highly expensive computational cost and give no rigorous theoretical guarantees. 
New techniques for efficient and principled neural architecture optimization can significantly advance the-state-of-the-art of deep learning.  

Recently, \citet{splitting2019} proposed
a \emph{splitting steepest descent (S2D)} method 
for 
 efficient neural architecture optimization, which frames the joint  optimization of the parameters and neural architectures into a continuous optimization problem in an infinite dimensional model space, and derives a computationally efficient (functional) 
 steepest descent procedure for solving it.   
 Algorithmically, S2D works by alternating between typical parametric updates with the architecture 
 fixed, and an architecture descent 
which  grows the neural network structures by optimally splitting critical neurons into a convex combination of multiple copies. 

 In S2D, the optimal rule for picking what neurons to split and how to split 
 is theoretically derived to yield the fastest descent of the loss in an asymptotically small neighborhood. Specifically, 
 the optimal way to split a neuron is to divide it into \emph{two equally weighted copies} along the 
 minimum 
 eigen-direction of a key quantity called \emph{splitting matrix} for each neuron.  
 Splitting a neuron into more than two copies \emph{can not} introduce any additional gain  theoretically and do not need to be considered for computational efficiency.
 %
 Moreover, the change of loss resulted from splitting a neuron 
 equals the minimum eigen-value  of its splitting matrix (called the \emph{splitting index}).
 Therefore, neurons whose splitting matrix is positive definite are considered to be ``\emph{splitting stable}'' (or not \emph{splittable}) in that 
 splitting them in any fashion can increase the loss and hence would be ``pushed back'' by subsequent gradient descent. 
 %
 In this way, the role of splitting matrices is analogous to how Hessian matrices characterize local stability for typical parametric optimization,  
and the local stability due to positive definite splitting matrices 
can be viewed as a notation of 
local optimality in the parameter-structure joint space. 
Unfortunately, the presence of the splitting stable status leads to 
a key limitation of the practical performance of splitting steepest descent, 
since the loss can be stuck at a relatively large value when the splittings become stable and can not continue. 
 This work 
 fills a surprising missing piece of  
 the picture outlined above, which allows us to address the local optimality problem in splitting descent with a simple algorithmic improvement.   
 We show that the notation of splitting stable caused by positive definite splitting matrices
 is in fact an artifact of splitting neurons into \emph{positively weighted} copies. 
 By simply considering \emph{signed splittings} which allows us to split neurons into copies  with both positive and negative weights,  
 the loss can continue decrease unless the splitting matrices equals zero for all the neurons. Intriguingly, the optimal spitting rules with signed weights 
 can have upto three or four copies (a.k.a. triplet and quartet splittings; see Figure~\ref{fig:vis}(c-e)), even though signed binary splittings (Figure~\ref{fig:vis}(a-b)), which introduces no additional neurons over the original positively weighted splitting, can work sufficiently well in practice. 

Our main algorithm, 
\emph{signed splitting steepest descent (S3D)}, 
which outperforms the original S2D in 
both 
theoretical guarantees and empirical performance.  
Theoretically, it yields stronger notion of optimality and allows us to establish convergence analysis that was impossible for S2D. 
Empirically,
S3D can learn smaller and more accurate networks in a variety of challenging benchmarks,  including CIFAR-100, ImageNet, ModelNet40, on which S3D 
substantially outperforms S2D and a variety of baselines for learning small and energy-efficient networks  \citep[e.g.][]{liu2017learning, li2016pruning, gordon2018morphnet, he2018amc}. 


\begin{figure*}[t]
    \centering
    \includegraphics[width=1.0\textwidth]{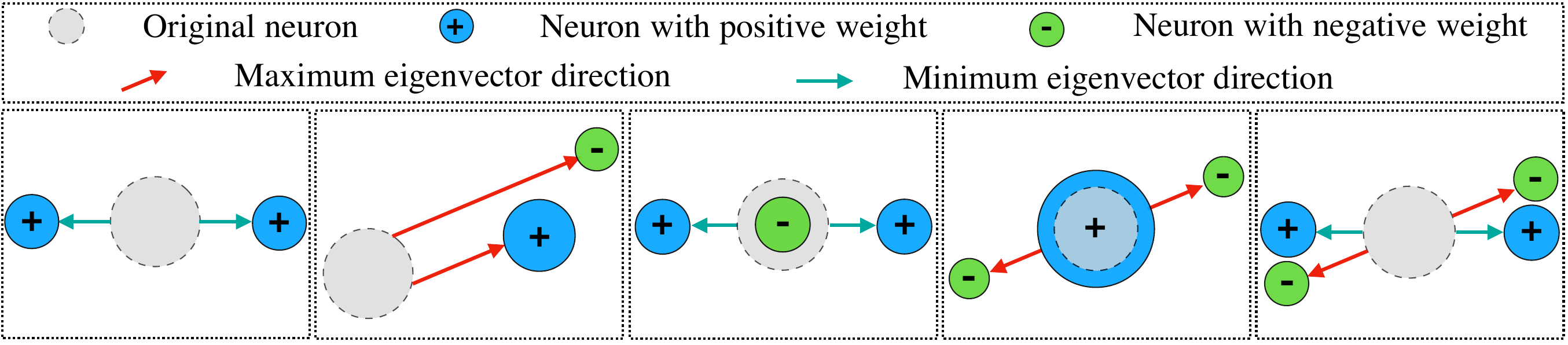}
\newcommand{\tmpfonthere}{\scriptsize}
\setlength{\tabcolsep}{1pt}
\begin{tabular}{lllll}
    \hspace{-1em}{\tmpfonthere (a) Positive Binary Splitting} & 
  {\tmpfonthere (b) Negative Binary Splitting}
 &
 {\tmpfonthere (c) Positive Triplet Splitting} &
 {\tmpfonthere (d) Negative Triplet Splitting} & 
 {\tmpfonthere (e) Quartet Splitting} 
    \end{tabular}
    \vspace{-.5em}
    \caption{Different splitting strategies. 
    The original splitting steepest descent \citep{splitting2019} 
    only uses positive binary splitting (a), which splits the neuron into two positively weighted copies. 
    By  allowing the neurons to split into output weighted by negative weights, we derive a host of new splitting rules (b,c,d,e), which can descent the loss more optimally, and avoid the appearance of splitting stability due to positive splitting matrices. In practice, we find that the combination of positive and negative binary splittings (a-b)
 works sufficiently well.}
    \label{fig:demo}
\vspace{-1em}
\end{figure*}

\section{Background: Splitting Steepest Descent}
Following \citet{splitting2019}, we start with the case of splitting a single-neuron network $f(x) =\sigma(\theta, x)$, 
where $\theta \in \RR^d$ is the parameter and $x$ the input. 
On a data distribution $\mathcal D$, 
the loss function of $\theta$ is 
$$
L(\theta) = \E_{x\sim \mathcal D} \left [\Phi\left (\sigma(\theta, x)\right)\right], 
$$
where $\Phi(\cdot)$ denotes a nonlinear loss function. 

\begin{wrapfigure}{r}{.35\textwidth}
\vspace{-5\baselineskip}
    \includegraphics[width=0.34\textwidth]{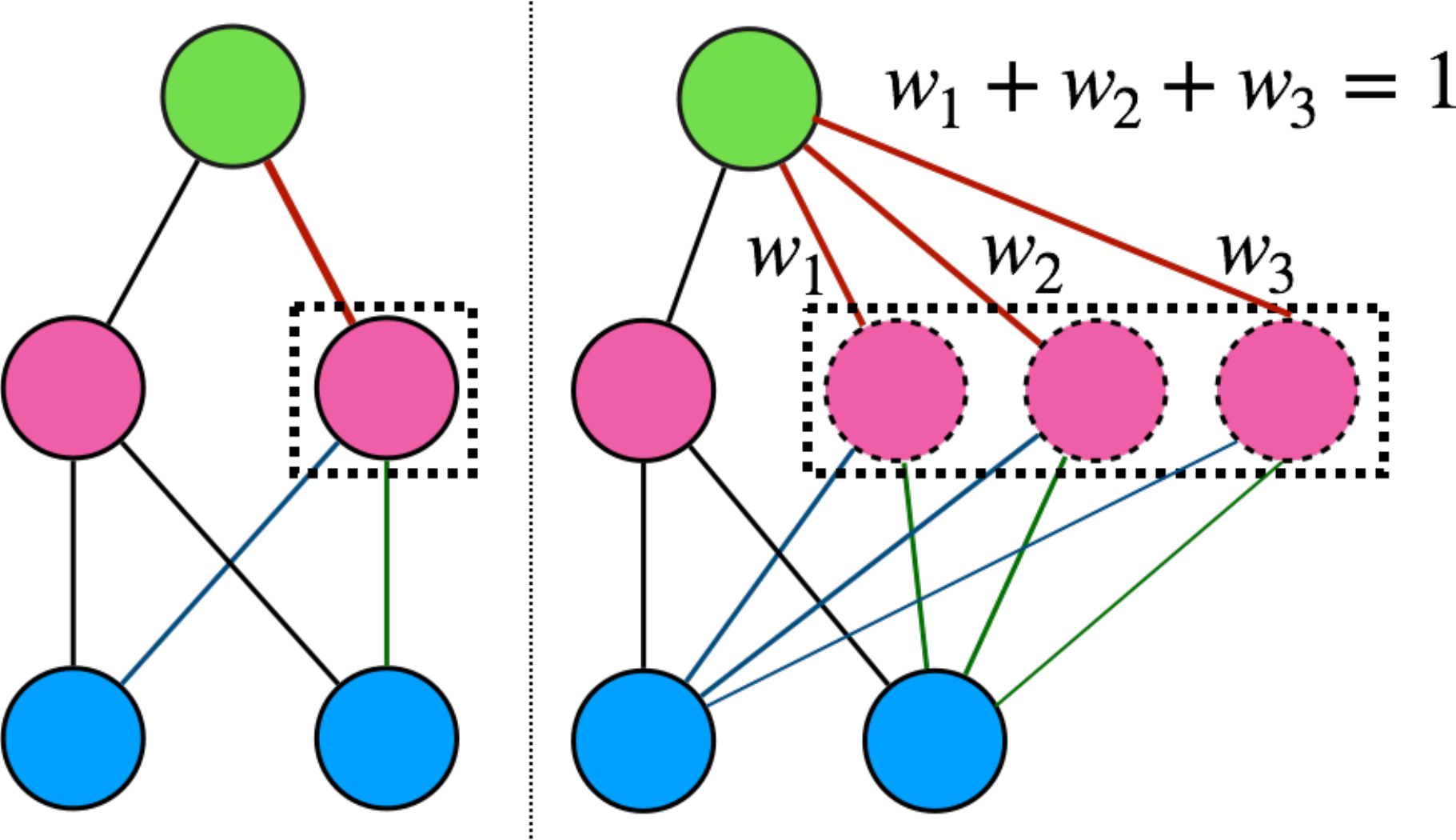} 
    \vspace{-1\baselineskip}
\end{wrapfigure}
Assume we split a neuron with parameter $\theta$ into $m$ copies whose parameters  are $
\{\theta_i\}_{i=1}^m$, each of which is associated with a weight $w_i \in \RR$,  yielding a large neural network of form  $f(x) = 
\sum_i w_i \sigma(\theta_i, x)$. 
Its loss function is 
%
$$
\Lm(\vv \theta, \vv w) = \E_{x\sim \mathcal D} \left [ \Phi\left (\sum_{i=1}^m w_i \sigma(\theta_i, x)\right )\right ], 
$$
where we write 
$\vv\theta=\{\theta_i\}_{i=1}^m$
$\vv w:=\{w_i\}_{i=1}^m$. 
We shall assume $\sum_i w_i = 1$, so that we obtain an equivalent network, 
or a \emph{network morphism} \citep{wei2016network}, 
when the split copies are not updated,  i.e., 
$\theta_i=\theta$ for $\forall i$. 
We want to find the optimal splitting scheme ($\vv\theta$, $\vv w$, $m$) to yield the  minimum loss $\Lm(\vv \theta, \vv w)$. 

Assume the copies $\{\theta_i\}$ can be decomposed into  $\theta_i = \theta + \epsilon( \delta_0 + \delta_i)$ where $\epsilon$ denotes a step-size parameter, 
$\delta_0 := \sum_i w_i \theta_i -\theta$ the average displacement of all copies (which implies $\sum_i w_i\delta_i = 0 $), and $\theta_i$ the individual ``splitting'' direction of $\theta_i$. 
\citet{splitting2019} showed the following key decomposition: 
\begin{align}\label{equ:decompLm}
\Lm(\vv\theta, \vv w) = L(\theta + \epsilon \delta_0) 
+ \frac{\epsilon^2}{2} \II(\vv\delta, \vv w; \theta) + O(\epsilon^3),
\end{align}
where $L(\theta + \epsilon \delta_0)$ denotes the 
effect of average displacement, corresponding to typical  parametric $\theta \mapsto \theta + \epsilon \delta_0$ without splitting, 
and $\II(\vv\delta, \vv w; ~\theta)$ denotes the effect of splitting the neurons; it is a quadratic form depending on a \emph{splitting matrix} defined in \citet{splitting2019}: 
\begin{align} \label{equ:splitting}
\begin{split} 
\II(\vv\delta, \vv w; ~ \theta)  = \sum_{i=1}^m w_i \delta_i^\top S(\theta) \delta_i, ~~~~~~~
 \text{where~~~~}S(\theta) &= \E_{x\sim \mathcal D} [\Phi'(\sigma(\theta, x)) \nabla_{\theta\theta}^2\sigma (\theta, x)]. 
\end{split}
\end{align}
Here $S(\theta) \in \RR^{d\times d}$ is  called the splitting matrix of $L(\theta)$. 
Because splitting increases the number of neurons and only contributes an $O(\epsilon^2)$ decrease of loss following \eqref{equ:decompLm}, 
it is preferred to
decrease the loss with typical parametric updates that requires no splitting (e.g., gradient descent),  
whenever the parametric local optimum of $L(\theta)$ is not achieved. 
However, when we research a local optimum of $L(\theta)$, 
splitting allows us to escape the local optimum at the cost of increasing the number of neurons. In \citet{splitting2019}, the optimal splitting scheme is framed into an optimization problem: 
\begin{align}\label{equ:optII}  
&
G_m^+ := \min_{\vv \delta, \vv w} 
\left\{ \II(\vv \delta, \vv w; \theta) \colon ~ 
\vv w\in  P_m^+, ~
\vv \delta \in {\Delta}_{\vv w}
\right\},
\end{align}
where we optimize 
the weights $\vv w$ in a probability simplex $P_m^+$ and 
splitting vectors $
\vv \delta$ in set set $\Delta_{\vv w}$: 
\begin{align}\label{equ:pm}
 P_m^+ = \big \{ \vv w \in\RR^m \colon ~ 
\sum_{i=1}^m w_i = 1, ~ w_j \geq  0,~~ \forall j \big \}, 
\end{align}
$$
\Delta_{\vv w}  = \big \{\vv\delta \in\RR^{m\times d} \colon 
\sum_{i=1}^m w_i \delta_i = 0,~ 
\norm{\delta_j}\leq 1, ~\forall j 
\big\}, 
$$
in which $\delta_i$ is constrained in the unit ball and the constraint $\sum_{i=1}^m w_i \delta_i = 0$ is to ensure a zero average displacement. 
\citet{splitting2019} showed that 
the optimal gain $G_m^+$ in \eqref{equ:optII}
depends on 
the minimum eigen-value $\lambda_{\min}$ of $S(\theta)$ in that 
$$G_m^+ = \min\left (\lambda_{\min},0\right).$$ 
If $\lambda_{\min} < 0$, 
we obtain a strict decrease of the loss, and the maximum decrease can be achieved by 
a simple binary splitting scheme ($m=2$),  in which the neuron 
is split into two equally weighted copies along the minimum eigen-vector direction $v_{\min}$ of $S(\theta)$, that is, 
\begin{align} \label{equ:binarysplitting}
m=2, && w_1 = w_2 = 1/2, &&\delta_1 = - \delta_2 =  v_{\min}. 
\end{align}
See Figure~\ref{fig:demo}(a) for an illustration.  
This binary splitting $(m=2)$ defines the best possible splitting in the sense of \eqref{equ:optII}, 
which means that it can not be further improved even when it is allowed to split the neuron into an arbitrary number $m$ of copies. 

On the other hand, if $\lambda_{\min}>0$, we have $G_m^+=0$ and the loss can not be decreased by any splitting scheme considered in \eqref{equ:optII}.  
This case 
was called being \emph{splitting stable} in \citet{splitting2019},  which means that even if the neuron is split into an  arbitrary number of copies  in arbitrary way (with a small step size $\epsilon$), all its copies would be pushed back to the original neuron when gradient descent is applied subsequently.  

\subsection{Main Method: Signed Splitting Steepest Descent}
\label{sec:method}
Following the derivation above, 
the splitting process would get stuck and stop when the splitting matrix $S(\theta)$ is positive definite 
($\lambda_{\min} >0$), and it yields small gain when $\lambda_{\min}$ is close to zero. 
Our key observation is that this phenomenon is in fact an artifact of  
constraining the weights $w_i$ to be non-negative in optimization \eqref{equ:optII}-\eqref{equ:pm}. 
By allowing negative weights, 
we can open the door to a much richer class of splitting schemes, which allows us to descent the loss more efficiently. 
Interestingly,
although
the optimal positively weighted splitting is always achievable by the binary splitting scheme ($m=2$) shown in \eqref{equ:binarysplitting},  
the optimal splitting schemes  with signed weights can be 
either binary splitting ($m=2$),  triplet splitting ($m=3$), or at most quartet splitting ($m=4$). 

Specifically, our idea is to replace \eqref{equ:optII} with 
\begin{align}\label{equ:optIIneg}   
&
\!\!\!\!\!\!G_m^{-c}
 := \min_{\vv \delta, \vv w} 
\left\{ \II(\vv \delta, \vv w; \theta) \colon ~
\vv w\in  P_m^{-c}, ~
\vv \delta \in {\Delta}_{\vv w}
\right\},
\end{align}
where 
the weight $\vv w$ is constrained in a larger set $P_m^{-c}$ whose size depends on a scalar  $c \in [1,\infty)$: 
\begin{align}\label{equ:pmneg}
\!\!\!\! P_m^{-c} = \big \{ \vv w \in \RR^m  \colon \sum_{i=1}^m w_i = 1, 
~\sum_{i=1}^m |w_i| \leq c \big \}.
\end{align}
We can see that 
$P_m^{-c}$ reduces to $P_m^+$ when $c = 1$, and contains negative weights when $c > 1$. By using $c > 1$, we enable a richer class of splitting schemes with signed  weights, hence yielding faster descent of the loss function.

The optimization in \eqref{equ:optIIneg} is more involved than the positive case \eqref{equ:optII}, but still yield elementary solutions. 
We now discuss the solution when we split the neuron into  
$m=2,3,4$ copies, respectively.  Importantly, we show that no additional gain can be made 
by splittings with more than $ m =4$ copies.

For notation, we denote by 
 $\lambda_{\min}$, $\lambda_{\max}$ the  smallest and largest eigenvalues of $S(\theta)$, respectively, and $v_{\min}$, $v_{\max}$ their corresponding eigen-vectors with unit norm.

\begin{thm}[\textbf{Binary Splittings}]\label{thm:2spliting}
For the optimization in  \eqref{equ:optIIneg} with $m=2$ and $c \geq 1$, 
we have 
$$
G_2^{-c} =   \min\left ( \lambda_{\min}, ~~ -\frac{c-1}{c+1}\lambda_{\max}, ~~ 0 \right ),
$$
and the optimum is achieved by
one of the following cases: 

i) no splitting ($\delta_1=\delta_2=0$), which yields $G_2^{-c} =0$; 

ii) the positive binary splitting in \eqref{equ:binarysplitting},  yielding $G_2^{-c} = \lambda_{\min}$;  

iii) the following ``negative'' binary splitting scheme:  
\begin{align}\label{equ:neg2splitting}
\begin{split} 
     w_1 = - \frac{c-1}{2}, ~~~~ \delta_1 =v_{\max}, ~~~~~~~~~~~~
         w_2 = \frac{c+1}{2}, ~~~~ \delta_2 = \frac{c-1}{c+1} v_{\max}. 
\end{split}
\end{align}
which yields $G_{2}^{-c}=-\frac{c-1}{c+1}\lambda_{\max}.$ 
This amounts to splitting the neuron into two copies with a positive and a negative weight, respectively, both of which move along the eigen-vector $v_{\max}$, but with different magnitudes (to ensure a zero average displacement).   See Figure~\ref{fig:demo}(b) for an illustration. 
\end{thm} 
Recall that the positive splitting  \eqref{equ:binarysplitting} follows the minimum eigen-vector 
$v_{\min}$, and achieves a decrease of loss only if $\lambdamin<0$. 
In comparison,  
the negative splitting  \eqref{equ:neg2splitting}
exploits  
the maximum eigen-direction $v_{\max}$ 
and achieves a decrease of loss when $\lambdamax >0$.  
Hence, unless $\lambdamin = \lambdamax=0$, or $S(\theta)=0$, a loss decrease can be achieved by either the positive or negative binary splitting. 


\begin{thm}[\textbf{Triplet Splittings}]\label{thm:3splitting}
For the optimization in \eqref{equ:optIIneg} with $m=3$ and $c \geq 1$, we have 
$$
G_3^{-c} =  \min\left (
\frac{c+1}{2}\lambda_{\min}, ~~-\frac{c-1}{2}\lambda_{\max},~~0 \right), 
$$
and the optimum is achieved by one of the following cases: 

i) no splitting ($\delta_1=\delta_2=\delta_3=0$), with $G_3^{-c} = 0$;

ii) the following ``positive'' triplet splitting scheme with two positive weights and one negative weights that yields $G_3^{-c} = \frac{c+1}{2}\lambdamin$: 
\begin{align}\label{equ:p3splitting}
\begin{array}{*3{>{\displaystyle}l}}
w_1 = \frac{c+1}{4},\ \ \ w_2 = \frac{c+1}{4},\ \ \ w_3 = - \frac{c-1}{2},\ \ \ \delta_1 = v_{\min},\ \ \ \delta_2 = -v_{\min},\ \ \ \delta_3= 0.
\end{array}
\end{align}
iii) the following ``negative'' triplet  splitting scheme with two negative weights and one positive weights that yields $G_3^{-c} = -\frac{c-1}{2}\lambdamax$: 
\begin{align}\label{equ:n3splitting}
\begin{array}{*3{>{\displaystyle}l}}
w_1 = -\frac{c-1}{4},  
& w_2 = -\frac{c-1}{4}, 
& w_3 =  \frac{c+1}{2},\ \ \ \delta_1 = v_{\max},\ \ \ \delta_2 = -v_{\max},\ \ \ \delta_3= 0,
\end{array}
\end{align}
\end{thm} 

Similar to the binary splittings, 
the positive and negative triplet splittings exploit the minimum and maximum eigenvalues, respectively. 
In both cases, the triplet splittings achieve larger descent than the binary counterparts, which is made possible by  placing a copy with no movement ($\delta_3=0$) to allow the other two copies to achieve larger descent with a higher degree of freedom. 

See Figure~\ref{fig:demo}(c)-(d) for illustration of the triplet splittings.  
Intuitively, the triplet splittings can be viewed as giving birth to two off-springs while keeping the original neuron alive, while the binary splittings ``kill'' the original neuron and only keep the two off-springs.  

We now consider 
the optimal quartet splitting ($m=4$), and show that no additional gain is possible with  $m\geq 4$ copies. 
\begin{thm}[\textbf{Quartet Splitting and Optimality}] 
\label{thm:4spliting}
For any $m\geq 4$, $m\in \mathbb Z_+$ and $c\geq 1$, we have 
$$
G_m^{-c}= G_4^{-c} =  
\frac{c+1}{2}\lambda^{\mathrm{th}}_{\min}~ -~ \frac{c-1}{2}\lambda^{\mathrm{th}}_{\max},
$$
where $\lambda_{\max}^{\mathrm{th}} =
\max(\lambda_{\max},0)$ and 
$\lambda_{\min}^{\mathrm{th}}=\min(\lambda_{\min},0)$. 
In addition, the optimum is achieved by the following splitting scheme with $m=4$: 
\begin{align}\label{equ:4splitting}
\begin{split} 
\w_{1}=\w_{2}=\frac{c+1}{4}, & \ \ \ \ \ \  \w_{3}=\w_{4}=-\frac{c-1}{4}, \ \ \ \ \ \ 
\delta_{1}=-\delta_{2}=
v_{\min}^{\mathrm{th}}, \ \ \ \ \ \  \delta_{3}=-\delta_{4}=
v_{\max}^{\mathrm{th}},
\end{split}
\end{align}
where $v_{\min}^{\mathrm{th}}:=
\ind_{[\lambdamin<0]} \times  v_{\min}$, ~  
$v_{\max}^{\mathrm{th}}:=\ind_{[\lambdamax>0]} \times  v_{\max}$, and $\ind_{[\cdot]}$ denotes the indicator function. 
\end{thm}

Therefore, if $\lambdamin = \lambdamax =0$, 
we have $v_{\min}^{\mathrm{th}}=v_{\max}^{\mathrm{th}} = 0$, and \eqref{equ:4splitting} yields effectively no splitting ($\delta_i=0$, $\forall i\in[4]$). In this case, no decrease of the loss can be made by  any splitting scheme, regardless of how large $m$ is. 

If $\lambdamax \geq \lambdamin > 0$ (resp. $\lambdamin \leq \lambdamax <0$), 
we have $v^{\mathrm{th}}_{\min} =0$ (resp. $v^{\mathrm{th}}_{\max}=0$), and \eqref{equ:4splitting} reduces to the positive (resp. negative) triplet splitting in Theorem~\ref{thm:3splitting}. 
There is no additional gain to use $m=4$ over $m=3$. 

If $\lambdamin<0<\lambdamax$, 
this yields a quartet splitting  (Figure~\ref{fig:demo}(e)) 
which has two positively weighted copies  split along the $v_{\min}$ direction,
and two negative weighted copies along the $v_{\max}$ direction. 
The advantage of this quartet splitting is that it exploits \emph{both maximum and minimum eigen-directions} simultaneously, while any binary or  triplet splitting can only 
benefit from one of the two directions.

\begin{algorithm*}[t] 
\caption{Signed Splitting Steepest Descent (S3D) for Progressive Training of Neural Networks} 
\begin{algorithmic} 
\STATE Starting from a small initial neural network.
Repeat the following steps until a convergence criterion is reached:  
\STATE \textbf{1. Parametric Updates}:
    Optimize the neuron weights using standard optimizer (e.g., stochastic gradient descent) to reach a local optimum, on which the parametric update can not make further improvement.  
    \vspace{.3\baselineskip} 
    
    \STATE \textbf{2. Growing by Splitting}: 
    Evaluate the maximum and minimum eigenvalues of each neuron; 
    select a set of neurons with most negative values of $G_{m}^{-c}$ with $m=2$, 3, or 4, using a heuristic of choice, and split these neurons using the optimal schemes 
    specified in Theorem~\ref{thm:2spliting}-\ref{thm:4spliting}. 
\end{algorithmic}
\label{alg:main}  

\end{algorithm*}

\textbf{Remark}\space\space 
A common feature of all the splitting schemes above ($m =2,3$ or $4$) 
is that the decrease of loss is all proportional to the spectrum radius of splitting matrix, that is,  
\begin{align} \label{equ:gmcrho}
G_m^{-c} \leq -\kappa_m \rho(S(\theta)), ~~~~~~~~\text{where}~~~~~~~~~~
\rho(S(\theta)) := \max(|\lambda_{\max}(S(\theta))|, ~|\lambda_{\min}(S(\theta))|),
\end{align}
where $\kappa_2 = \frac{c-1}{c+1}$ and $\kappa_m = \frac{c-1}{2}$ for $m \geq 3$.  Hence, 
unless $\rho(S(\theta))=0$, which implies $S(\theta)=0$, 
we can always decrease the loss by the optimal splitting schemes with any $m\geq 2$. 
This is in contrast with the optimal positive splitting in \eqref{equ:binarysplitting}, 
which  get stuck when $S(\theta)$ is positive semi-definite ($\lambdamin \geq0$). 

We can see from Eq~\eqref{equ:gmcrho}   that the effects of splittings with different $m\geq 2$ are qualitatively similar. 
The improvement of using the triplet and quartet splittings over the binary splittings is  only up to 
a constant factor of $\kappa_3/\kappa_2 = (c+1)/2$, and may not yield a significant difference on the final optimization result. 
As we show in experiments, it is preferred to use binary splittings ($m=2$), as it introduces less neurons in each splitting and yields much smaller neural networks. 

\paragraph{Algorithm}
Similar to \citet{splitting2019},
the splitting descent can be easily extended 
to general neural networks with multiple neurons, possibly in different layers, 
because the effect of splitting different neurons are additive as shown in Theorem~2.4 of \citet{splitting2019}. 

This yields the practical algorithm in \eqref{alg:main}, 
in which we alternate between 
\emph{i}) the  standard parametric update phase, 
in which we use traditional gradient-based optimizers 
until 
no further improvement can be made by pure parametric updates, 
and \emph{ii}) the splitting phase, in which we evaluate the minimum and maximum eigenvalues of the splitting matrices of the different neurons, select a subset of neurons with the most negative values of  $G_m^{-c}$ with $m=2$, 3, or 4, and split these neurons using the optimal schemes specified in Theorem~\ref{thm:2spliting}-\ref{thm:4spliting}.  

The rule for deciding how many neurons to split at each iteration 
can be a heuristic of users' choice. 
For example, we can decide a maximum number  $k$  of neurons to split and a positive threshold $\eta$, 
and select the top $k$ neurons with the most negative values of $G_m^{-c}$,
and satisfy $G_{m}^{-c} \leq -\eta$.  

\paragraph{Computational Cost}
Similar to \citet{splitting2019}, the eigen-computation of signed splittings requires  $\obig(md^3)$ in time and $\obig(md^2)$ in space, where $m$ is the number of neurons and $d$ is the parameter size of each neuron. 
However, this can be significantly improved by using 
the Rayleigh-quotient gradient descent for eigen-computation 
introduced in \citet{wang2019energy}, which 
has roughly the same time and space complexity as typical parametric back-propagation on the same network (i.e., $\obig(md^2)$ in time and $\obig(md)$ in space). 
See Appendix \ref{appendix:Fast_RQ} for more details on how we apply Rayleigh-quotient gradient descent in signed splittings. 

\paragraph{Convergence Guarantee} 
The original  S2D does not have a rigours convergence guarantee due to the local minimum issue associated with positive definite splitting matrices. 
With more signed splittings, 
the loss can be minimized much more thoroughly, 
and hence allows us to establish provably convergence guarantees. 
In particular, we show that, under proper conditions, 
by splitting two-layer neural networks using  
S3D with only binary splittings starting from a single-neuron network, 
we achieve a training MSE loss of $\eta$ by splitting at most $\mathcal O((n/(d\eta))^{3/2})$ steps, 
where $n$ is data size and $d$ the dimension of the input dimension. 
The final size of the network we obtain, which equals $\mathcal O((n/(d\eta))^{3/2})$, is smaller than the number of neurons required 
for over-parameterization-based analysis of standard gradient descent training of neural networks, 
and hence provides a theoretical justification of that splitting can yield accurate and smaller networks than standard gradient descent. 
For example, the analysis in \citet[][]{du2018gradient} requires   $\mathcal O(n^6)$ neurons, or $\mathcal O(n^2/d)$ in \citet[][]{oymak2019towards},  
larger than what we need when $n$ is large. 
The detailed results are shown in Appendix due to space constraint. 

\section{Experiments} 
We test our algorithm on various benchmarks, including CIFAR-100, ImageNet and ModelNet40. We apply our signed splitting steepest descent (S3D) following Algorithm \ref{alg:main} and compare it with splitting steepest descent (S2D) \citep{splitting2019}, which is the same as Algorithm~\ref{alg:main} except that only positive splittings are used.
We also consider an energy-aware variant following \citet{wang2019energy}, 
in which the increase of energy cost for splitting each neuron is estimated at each splitting step, 
and the set of neurons to split is selected by 
solving a knapsack problem 
to maximize the total splitting gain subject to a constraint on the increase of energy cost. See  \citet{wang2019energy} for details.

We tested S3D with different splitting sizes ($m=2,3,4$) and found that the binary splitting ($m=2$) tends to give the best performance in practical deep learning tasks of image and point cloud classification.  
This is because $m=3,4$ tend to give much larger networks while do not yield significant improvement over $m=2$ to compensate the faster growth of network size. 
In fact, if we consider the average gain of each new copy, $m = 2$ provides a better  trade-off between the accuracy and network sizes. Therefore, we only consider $m = 2$  in all the deep learning experiments.
Due to the limited space, we put more experiment details in Appendix.

\paragraph{Toy RBF neural networks}
We revisit the toy RBF neural network experiment described in \citep{splitting2019} to domenstrate the benefit of introducing signed splittings.  
\citet{splitting2019}, it still tends to get stuck at local optima  when the splitting matrices are positive definite. By using more general signed splittings, our S3D algorithm allows us to escape the local optima that S2D can not escape, hence yielding better results. 
For both S2D and S3D, we start with an initial network with a single neuron and gradually grow it by splitting neurons. 
We test both S2D which includes only positive binary splittings, 
and S3D with signed binary splittings ($m=2$), 
triplet splittings ($m=3$), and quartet splittings ($m=4$), respectively. More experiment setting can be found in Appendix \ref{app:toy}.

As shown in Figure \ref{fig:nn_toy} (a),  
S2D gets stuck in a local minimum, 
while our signed splitting can escape the local minima and fit the true curve well in the end. Figure \ref{fig:nn_toy} (b) shows different loss curves trained by S3D ($m=2$) with different $c$. 
The triangle remarks in  Figure \ref{fig:nn_toy} (b) 
indicate the first time when positive and signed splittings pick  differ neurons.  
%
Figure \ref{fig:nn_toy} (d) further provides evidence showing that S3D can pick up a different but better neuron (with large $\lambda_{\max}$) to split compared with S2D, which helps the network get out of local optima.
\begin{figure*}[t]
\centering
\setlength{\tabcolsep}{1.0pt}
\begin{tabular}{ccccc}

\raisebox{3.0em}{\rotatebox{90}{{\scriptsize Loss}}}~~\includegraphics[width =0.2\textwidth]{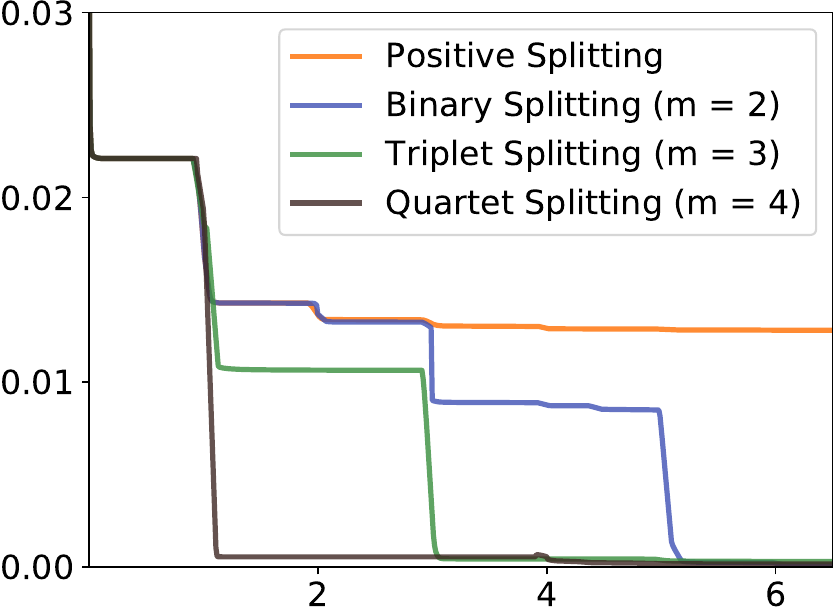}&
\raisebox{3.0em}{\rotatebox{90}{{\scriptsize Loss}}}\hspace{-0.2em}~\includegraphics[width =0.2\textwidth]{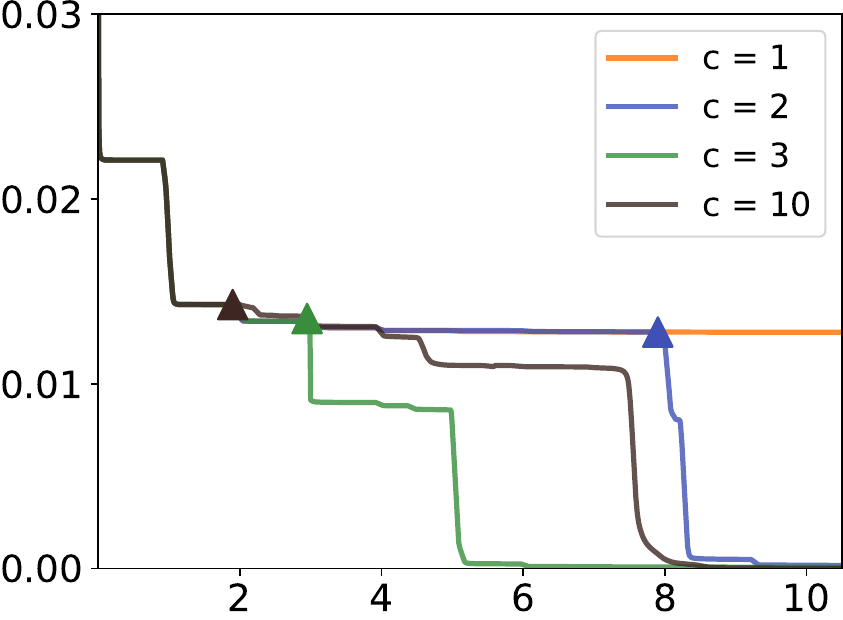} &
\includegraphics[width =0.185\textwidth]{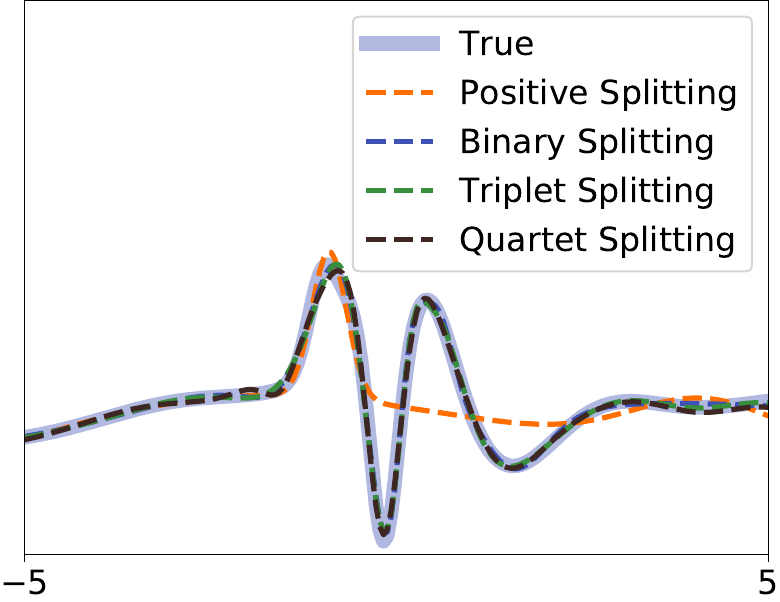} & 
\raisebox{-0.10em}{\includegraphics[width =0.2\textwidth]{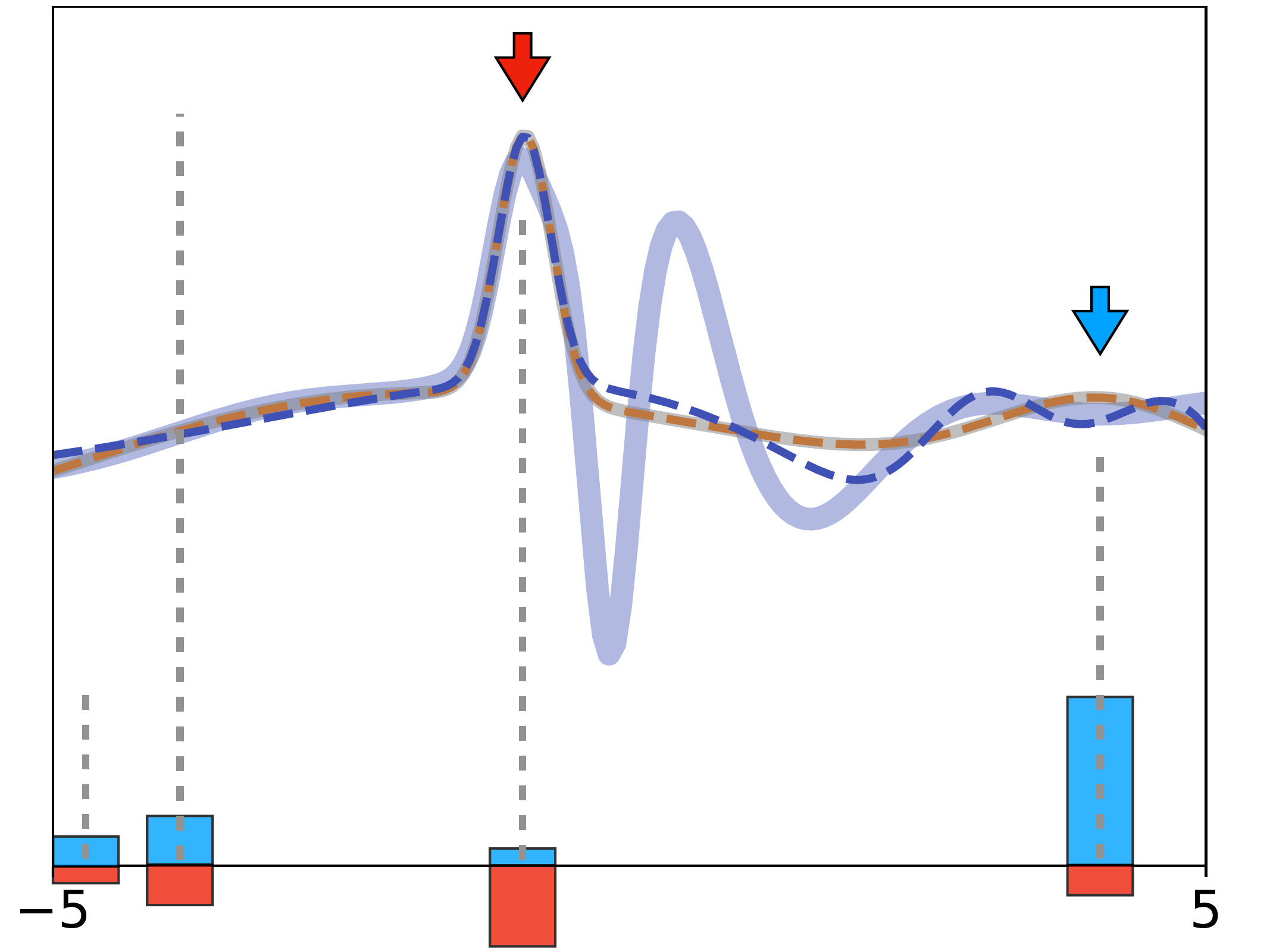}} & 
\raisebox{0.50em}{\includegraphics[width =0.15\textwidth]{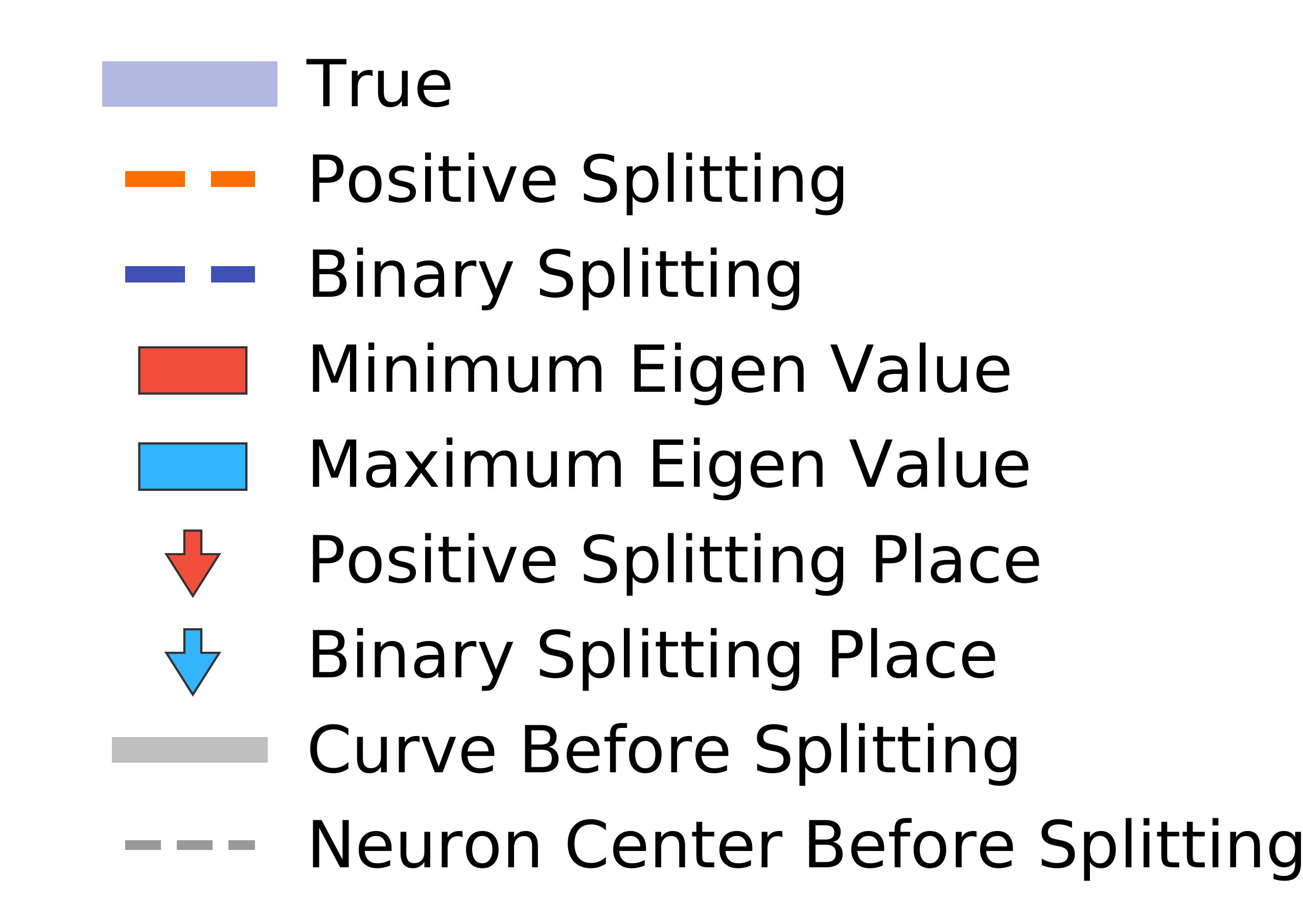} }
\\
\scriptsize ~~~\quad Iterations $(\times 10^{4})$
&
\scriptsize ~~~\quad Iterations $(\times 10^{4})$
&
&
\\
\scriptsize (a) & \scriptsize (b) & \scriptsize (c) & \scriptsize (d)

\end{tabular}
\vspace{-1em}
\caption{\small
Results on a one-dimensional RBF network.
(a) Loss curve of different splitting methods when $c=3$. (b) Loss curves of signed binary splittings ($m=2$) with different values of $c$. 
Note that $c=1$ reduces to positive splitting. The triangle markers indicate the first time when S2D and S3D give different splitting results. (c)  The curve fitted by different splitting methods when the network grow into 5 neurons. 
(d) The centers of the RBF neurons (indicated by the bar centers) of the curve learned when 
applying signed binary splittings ($m=2$) for 4 steps, 
and their corresponding maximum and minimum eigenvalues (the blue and red bars). 
The blue and red arrows indicate the location of splitting according to the maximum and minimum eigenvalues, respectively, 
and the blue and orange dashed lines are the corresponding curves of binary splittings and positive splittings $(m=2)$ we obtained after the splittings.}
\label{fig:nn_toy}
\vspace{-1.0em}
\end{figure*}

\begin{figure*}[ht]
\centering
\scalebox{0.9}{
\setlength{\tabcolsep}{1.0pt}
\begin{tabular}{cccc}

\raisebox{3.0em}{\rotatebox{90}{{\scriptsize Test Accuracy}}}~~\includegraphics[height =0.20\textwidth]{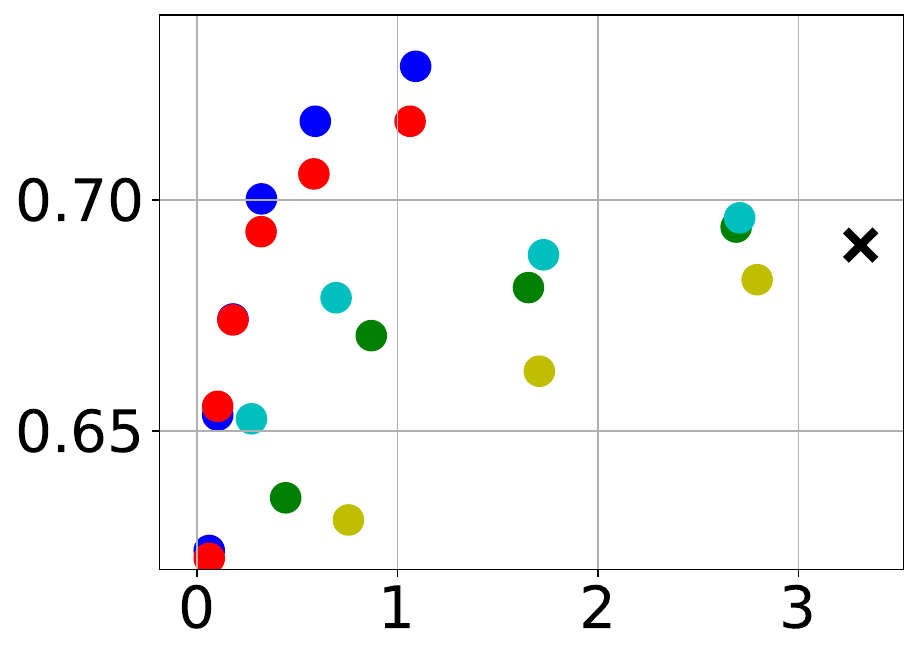}
&\raisebox{3.0em}{\rotatebox{90}{{\scriptsize Test Accuracy}}}~~\includegraphics[height=0.20\textwidth]{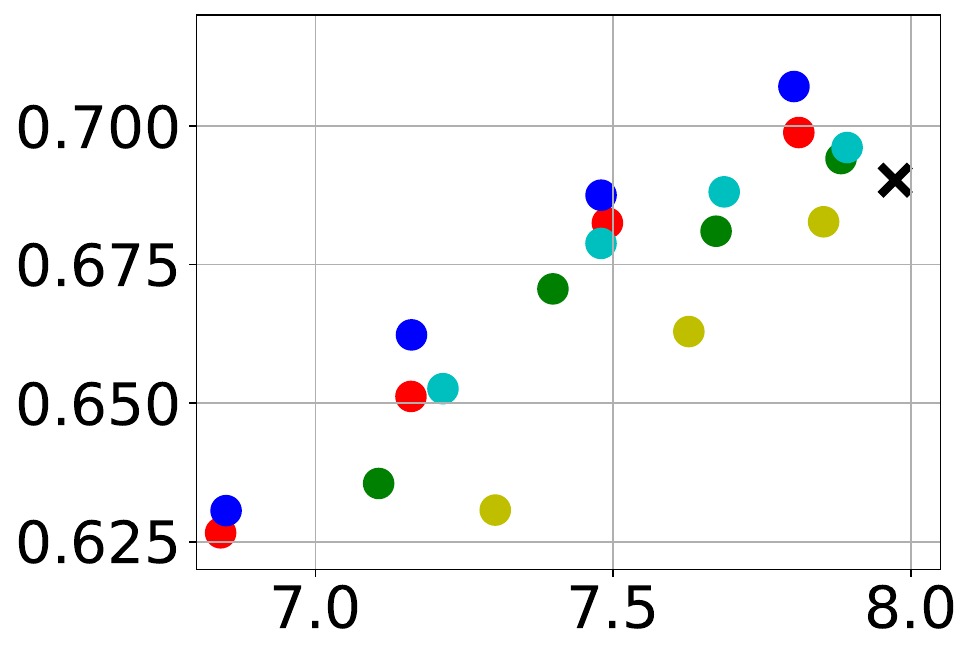} 
& \raisebox{1em}{\includegraphics[height =0.10\textwidth]{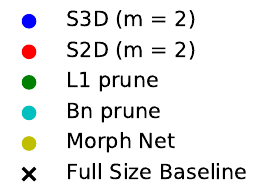}}&
\raisebox{3.0em}{\rotatebox{90}{{\scriptsize Test Accuracy}}}\includegraphics[height =0.20\textwidth]{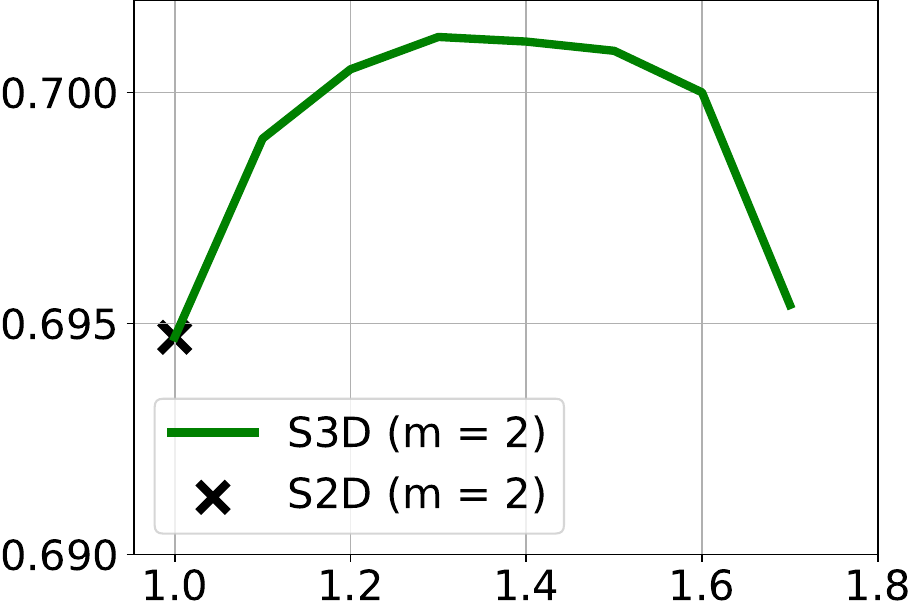}\\
\scriptsize ~~\quad\quad(a) \# Params (M)
&\scriptsize ~~\quad\quad(b) Log10 (FLOPs)
& 
& \scriptsize \quad\quad(c) Value of c 
\end{tabular}
}
\vspace{-1em}
\caption{\small 
Results of MobileNetV1 on CIFAR-100. 
(a) Testing accuracy v.s. the number of parameters of the models learned by S3D with binary splittings ($m=2$) and other baselines. 
(b) Results in 
the energy-aware setting by S3D with binary splittings ($m=2$) and other baselines. (c) Results of S2D and S3D ($m=2$) when $c$ varies 
in the same setting as that in (a) at the $5$th spitting step. 
}
\label{fig:cifar100}
\end{figure*}

\paragraph{Results on CIFAR-100} 
We apply  S3D  to grow DNNs for the image classification task. We test our method on MobileNetV1 \citep{howard2017mobilenets} on CIFAR-100 and compare our S3D with S2D \citep{splitting2019} as well as other pruning baselines, including L1 Pruning \citep{liu2017learning}, Bn Pruning \citep{liu2017learning} and MorphNet \citep{gordon2018morphnet}. 
We also apply our algorithm in an energy-aware setting discussed in \citet{wang2019energy}, 
which decides the best neurons to split 
by formulating a knapsack problem  to best trade-off  the splitting gain and energy cost; see  \citet{wang2019energy}
for the details.  
To speedup the eigen-computation 
in S3D and S2D, we use the fast gradient-based eigen-approximation algorithm in \citep{wang2019energy} (see Appendix~\ref{appendix:Fast_RQ}).  
Our results show that our algorithm outperforms prior arts with higher accuracy and lower cost in terms of both model parameter numbers and FLOPs. See more experiment detail in Appendix \ref{app:cifar}

Figure \ref{fig:cifar100} (a) and (b) show that our S3D algorithm outperforms all the baselines in both the standard-setting and
the energy-aware setting of \citet{wang2019energy}. Table \ref{table:cifar100} reports the testing accuracy,  parameter size and FLOPs of the learned models. We can  see that our method achieves significantly higher accuracy as well as lower  parameter sizes and FLOPs. 
%
 We study the relation between testing accuracy and the hyper-parameter $c$ in Figure \ref{fig:cifar100} (c), at the 5th splitting step in Figure~\ref{fig:cifar100} (a) (note that  $c=1.0$ reduces to S2D).  We can see that $c\approx 1.3$ is optimal in this case.




\begin{table}[t]

\parbox{.5\linewidth}{\begin{center}
\setlength{\tabcolsep}{5pt}
\renewcommand\arraystretch{1.0}
\vspace{-11.0em}
\scalebox{0.8}{

\begin{tabular}{l|ccc}
\hline
Method & Accuracy & \# Param (M) & \# Flops (M) \\
\hline
Full Size Baseline& 69.04& 3.31 & 94.13\\
L1 \citep{liu2017learning} & 69.41& 2.69 & 76.34\\
Bn \citep{liu2017learning} & 69.61& 2.71 & 78.15\\
MorphNet \citep{gordon2018morphnet}& 68.27& 2.79 & 71.63\\
\hline
S2D-5 \citep{splitting2019} & 69.69& 0.31 & 79.94\\
S3D-5 & \textbf{70.19}& \textbf{0.30} & 73.69\\
\hline

\end{tabular}
}
\end{center}

\caption{{Comparison of different methods when the testing accuracy is around 69\%.}
S2D-5 and S3D-5 represent applying S2D and S3D $(m=2)$ for  5 splitting steps, respectively. 
}\label{table:cifar100}

}
\quad
\parbox{.45\linewidth}{
\setlength{\tabcolsep}{5pt}
\renewcommand\arraystretch{1.2}
\vspace{-0.5em}
    \centering
   
    \scalebox{0.8}{
    \begin{tabular}{l|ccc}
        \hline
        Model & MACs (G) & Top-1 & Top-5 \\
        \hline \hline
        MobileNetV1 (1.0x) & 0.569 & 72.93 & 91.14 \\
        S2D-4  &  0.561 & 73.96 & 91.49  \\
        S3D-4  & \textbf{0.558} & \textbf{74.12} & \textbf{91.50} \\
        \hline \hline
        MobileNetV1 (0.75x) & 0.317 & 70.25 & 89.49\\
        AMC \citep{he2018amc} & 0.301 & 70.50 &	89.30\\
        S2D-3  & 0.292 & 71.47 & 89.67 \\
        S3D-3 & \textbf{0.291} & \textbf{71.61} & \textbf{89.83} \\
        \hline \hline
        MobileNetV1 (0.5x) &  0.150 & 65.20 & 86.34\\
        S2D-2  &  \textbf{0.140} & 68.26 & 87.93 \\
        S3D-2 & \textbf{0.140} & \textbf{68.72} & \textbf{88.19} \\
         \hline \hline
         S2D-1   & 0.082 & 64.06 & 85.30 \\
         S3D-1  & 0.082 & \textbf{64.37} & \textbf{85.49} \\
         \hline \hline
         Seed &  0.059 & 59.20 & 81.82 \\ 
         \hline
    \end{tabular}
    }
    \vspace{1em}
    \caption{Results of ImageNet classification using MobileNetV1. S2D-$k$ and S3D-$k$ denote we split the network $k$ times using S2D and S3D ($m=2$), respectively.}
    \label{tab:imagenet_mbv1}
    }
\end{table}

\begin{table}[t]
\setlength{\tabcolsep}{3pt}
\vspace{-11em}
\parbox{0.5\textwidth}{
\scalebox{.8}{
\begin{tabular}{l|ccc}
\hline
Model & Acc. & Forward time (ms) & \# Param (M)\\
\hline
PointNet \citep{Qi_2017_CVPR} & 89.2 & 32.19 & 2.85\\
PointNet++ \citep{NIPS2017pointnet++}& 90.7& 331.4 & 0.86\\
\hline
DGCNN (1.0x)& 92.6& 60.12 & 1.81\\
DGCNN (0.75x)& 92.4& 48.06 & 1.64\\
DGCNN (0.5x)& 92.3& 38.90 & 1.52\\
\hline
DGCNN-S2D-4 & 92.7 & 42.83 & 1.52 \\
DGCNN-S3D-4 & \textbf{92.9} & 42.06 & 1.51\\
\hline

\end{tabular}}

\caption{Results on the ModelNet40 classification task. DGCNN-S2D-4 and DGCNN-S3D-4 
denote  applying S2D and S3D ($m=2$)  
for 4 splitting steps, respetively.} 
\label{tab:dgcnn}
\vspace{-2em}
}
\end{table}

\paragraph{Results on ImageNet} 
We apply our method in ImageNet classification task. 
We follow the setting of \citep{wang2019energy}, 
using their energy-aware neuron selection criterion and 
fast gradient-based eigen-approximation. 
We also compare our methods with AMC \citep{he2018amc} , full MobileNetV1 and MobileNetV1 with $0.75\times$, $0.5\times$ width multipliers on each layers. 
We find that 
our S3D achieves higher Top-1 and Top-5 accuracy than other methods with comparable multiply-and-accumulate operations (MACs). See details of setting in Appendix \ref{app:img}.
%
Table \ref{tab:imagenet_mbv1} shows that our S3D obtains better Top-1 and Top-5 accuracy compared with the S2D in \citet{wang2019energy} and other baselines 
 with the same or smaller MACs.
 We also visualize the filters after splitting on ImageNet; see Appendix \ref{sec:vis}. 

\paragraph{Results on Point Cloud Classification}  
\myempty{ 
Point cloud is a simple and popular representation of 3D objects, which can be easily captured and processed by mobile devices. 
Point cloud classification amounts  to 
classifying 3D objects based on their point cloud  representations, and is found in many cutting-edge AI applications, such as face recognition in Face ID and LIDAR-based recognition in autonomous driving. 
Since many of these applications are deployed on mobile devices, 
a key challenges is to build small and energy efficient networks with high accuracy. 
We can attack this challenge with splitting steepest descent. 
}
%
 We consider  point cloud classification with 
 Dynamic graph convolution neural network (DGCNN) \citep{wang2019dynamic}. 
DGCNN one of the best networks for point cloud, but tends to be expensive in both speed and space,   because it involves K-nearest-neighbour (KNN) operators 
 for aggregating neighboring features on the graph.  
 We apply S3D to search better DGCNN structures 
 with smaller sizes, hence significantly improving the space and time efficiency. 
 Following the experiment in \citet{wang2019dynamic}, we choose ModelNet40 as our dataset. See details in Appendix \ref{app:dgcnn}. 
%
%
 %
 %
 Table \ref{tab:dgcnn} shows the result  compared with PointNet \citep{Qi_2017_CVPR}, PointNet++ \citep{NIPS2017pointnet++} and DGCNN with different multiplier on its EgdConv layers. We compare the accuracy as well as model size and time cost for forward processing.
  For forward processing time, we test it on a single NVIDIA RTX 2080Ti  with a batch size of 16. We can see that our S3D algorithm obtains networks with the highest accuracy among all the methods, with a faster forward processing speed than DGCNN ($0.75\times$) and a smaller model size than DGCNN ($0.5\times$). 

 



    


\section{Related Works}
 Neural Architecture Search (NAS)  has been traditionally framed as a discrete combinatorial  optimization 
 and solved based on black-box optimization methods such as reinforcement learning \citep[e.g.][]{zoph2016neural,zoph2018learning}, 
 evolutionary/genetic algorithms \citep[e.g.,][]{stanley2002evolving, real2018regularized}, 
 or continuous relaxation followed with gradient descent  
 \citep[e.g.,][]{liu2018darts, xie2018snas}. 
 These methods need to search in a large model space with expensive evaluation cost,  
 and can be computationally expensive or easily stucked at local optima. Techniques such as   
  weight-sharing \citep[e.g.][]{pham2018efficient, cai2018proxylessnas,bender2019understanding}  
  and low fidelity estimates \citep[e.g.,][]{zoph2018learning, pmlr-v80-falkner18a, runge2018learning} 
  have been developed to alleviate the cost problem in NAS; 
  see e.g., \citet{elsken2019neural, wistuba2019survey} for recent surveys of NAS. 
  In comparison, splitting steepest descent is based on a significantly different functional steepest view that leverages the fundamental topological information of deep neural architectures to enable more efficient search, 
  ensuring both rigorous theoretical guarantees and superior practical performance. 

 The idea of progressively growing neural networks has been considered 
 by researchers in various communities from different angles. 
 However, most existing methods are based on heuristic ideas. 
 For example, \citet{wynne1992node} proposed a heuristic method to split neurons based on the eigen-directions of covariance matrix of the gradient. See e.g., \citet{ghosh1994structural, utgoff1998constructive} for surveys of similar ideas in the classical literature. 
 
 Recently, \citet{chen2015net2net} proposed a method called Net2Net for knowledge transferring which grows a well-trained network by splitting randomly picked neurons along random directions.  
 Our optimal splitting strategies can be directly adapted to improve Net2Net.  
 Going beyond node splitting, more general operators  
 that grow networks while preserving the function represented by the networks, referred to as \emph{network morphism}, have been studied  and exploited in a series of recent works \citep[e.g.,][]{chen2015net2net, wei2016network, cai2018efficient,  elsken2018efficient}. 
 

 A more principled progressive training approach for neural networks 
 can be derived using Frank-Wolfe  \citep[e.g.,][]{schwenk2000boosting,bengio2006convex, bach2017breaking}, 
 which yields greedy algorithms that 
 iteratively add optimal new neurons while keeping the previous neurons fixed. 
 Although rigorous convergence rate can be established for these methods \citep[e.g.,][]{bach2017breaking}, they are not practically applicable because adding each new neuron requires to solve an intractable non-convex global optimization problem. 
  In contrast, the splitting steepest descent approach is fully computationally tractable, 
  because the search of the optimal node splitting schemes amounts to an tractable eigen-decomposition problem (albeit being non-convex). 
 The original S2D in \citet{splitting2019} did not provide a convergence guarantee, because the algorithm gets stuck when the splitting matrices become positive definite. 
 By using signed splittings, our S3D can escape more local optima, ensuring both strong theoretical guarantees and better empirical performance.

An alternative approach for learning small and energy-efficient networks is to \emph{prune} large pre-trained neural networks to obtain compact sub-network structures   \citep[e.g.,][]{han2015deep, li2016pruning, liu2017learning, liu2018rethinking, frankle2018lottery}. 
As shown in our experiments and \citet{splitting2019, wang2019energy}, 
the splitting approach can outperform existing pruning methods, without requiring the overhead of  pre-traininging large models. A promising future direction is to design algorithms that adaptively combine splitting with pruning to achieve  better results. 


\section{Conclusion}
In this work, 
we develop an extension of the splitting steepest descent framework  
to avoid the local optima by introducing signed splittings. 
Our S3D can learn small and accurate networks in challenging cases. 
%
For future work, we will develop further speed up of S3D and explore more flexible ways for optimizing network architectures going beyond neuron splitting. 
%

\section*{Acknowledge}
The work is conducted in the statistical learning and AI group in  computer science at UT Austin, which is supported in part by CAREER-1846421, SenSE-2037267, EAGER-2041327, and NSF AI Institute for Foundations of Machine Learning (IFML).


\bibliography{main}
\bibliographystyle{bibstyle}

\newpage\clearpage
\appendix
\onecolumn

\newpage
\section{Derivation of Optimal Splitting Schemes with Negative Weights}

\begin{lem}\label{lem:eigen}
Let $(\vv\delta^*, \vv w^*)$  be an optimal solution of \eqref{equ:optIIneg}. 
Then $\delta_i^*$ must be an eigen-vector of $S(\theta)$ unless $w_i^* = 0$ or $\delta_i^* =0$. 
\end{lem}
\begin{proof} 
Write $S = S(\theta)$ for simplicity. 
With fixed weights $\vv w$, the optimization w.r.t. $\vv \delta$ is
\begin{align*}
    \min_{\vv\delta} 
    \sum_{i=1}^m w_i \delta_i ^\top S \delta_i ~~~s.t.~~~ \norm{\sum_{i=1}^m w_i \delta_i} = 0,~~ \norm{\delta_i} = 1.
\end{align*}
By KKT condition, the optimal solution must satisfy 
\begin{align*}
    & w_i^* S \delta_i^*  - 
    \lambda_1 w_i^* \bar\delta^* - \lambda_2^* \delta_i = 0   \\
    & \bar \delta^* :=\sum_{i=1}^m w_i^* \delta_i^* = 0,
\end{align*}
where $\lambda_1 $ and $\lambda_2$ are two Lagrangian multipliers.  Canceling out $\bar \delta^*$ gives 
$$
w_i^*  S \delta_i^*  - \lambda_2 \delta_i^* =0.
$$
Therefore, if $w_i^*\neq 0$ and $\delta_i\neq 0$, then $\delta_i^*$ must be the eigen-vector of $S$ with eigen-value $\lambda_2/w_i^*$. 
\end{proof}

\newpage \clearpage
\subsection{Derivation of Optimal Binary Splittings ($m=2$)}

\begin{thm}
1) Consider the optimization in \eqref{equ:optIIneg} with $m =2$ and $c\geq 1$. Then the optimal solution must satisfy 
$$\delta_1 = r_1 v, ~~~~\delta_2 = r_2 v,$$
where $v$ is an eigen-vector of $S(\theta)$ and $r_1, r_2 \in \R^2$ are two scalars. 

2) In this case, the optimization reduces to 
\begin{align} \label{equ:R2Opt}
\begin{split} 
    G_2^{-c}:= 
    \min_{\vv w, \vvr, v} 
     (w_1 r_1^2 & + w_2 r_2^2)\times 
     \lambda \\ 
    s.t. ~~~&  w_1 + w_2 = 1 \\
    & w_1 r_1 + w_2 r_2 = 0 \\
    & |w_1| + |w_2|  \leq c \\
    & |r_1|, |r_2| \leq 1 \\
    & \text{$\lambda$ is an eigen-value of $S(\theta)$}. 
    \end{split}
    \end{align}
    
3)    
The optimal value above is 
\begin{align}\label{equ:G2capp}
G_{2}^{-c} = 
\min\left (  \lambdamin, ~~  -\frac{c-1}{c+1}\lambdamax, ~~
0 \right ). 
\end{align}
If $  -\frac{c-1}{c+1} \lambda_{\max} <  \min(\lambdamin,0)$, the optimal solution is achieved by 
\begin{align*} 
        w_1 = - \frac{c-1}{2}, ~~~~ \delta_1 =v_{\max}, &&  
         w_2 = \frac{c+1}{2}, ~~~~ \delta_2 = \frac{c-1}{c+1} v_{\max}. 
        \end{align*}
If $
\lambda_{\min} <  
\min(-\frac{c-1}{c+1}\lambdamax,0)$ 
the optimal solution is achieved by 
\begin{align*} 
        w_1 =  \frac{1}{2}, ~~~~ \delta_1 =v_{\min}, &&  
         w_2 = \frac{1}{2}, ~~~~ \delta_2 =  - v_{\min}. 
        \end{align*}
If $0\leq \min(\lambdamin,-\frac{c-1}{c+1}\lambdamax)$, and hence $\lambdamin=\lambdamax=0$,  
the optimal solution is achieved by no splitting: $\delta_1=\delta_2=0.$ 
\end{thm} 
\begin{proof}
1) The form of $\delta_1 = r_1v$ and $\delta_2 = r_2 v$ is immediately implied by the constraint 
$w_1 \delta_1 + w_2 \delta_2 = 0$. By Lemma~\ref{lem:eigen}, $v$ must be an eigen-vector of $S(\theta)$. 

2) Plugging $\delta_1 = r_1v$ and $\delta_2 = r_2 v$ into \eqref{equ:optIIneg} directly implies \eqref{equ:R2Opt}. 

3) 
Following \eqref{equ:R2Opt}, we seek to minimize the product of $t(\vv w, \vvr) := w_1 r_1^2 + w_2 r_2^2$ and $\lambda$.  
If $\lambda \geq 0$, we need to minimize $t(\vv w, \vvr)$, while if $\lambda \leq 0$, we need to maximize $t(\vv w, \vvr)$. 
Lemma~\ref{lem:solve2min} and \ref{lem:solve2max} below show that 
the minimum and maximum values of
$t(\vv w, \vvr)$ equal $-\frac{c-1}{c+1}$ and $1$, 
respectively. 
Because the range of $\lambda$ is $[ \lambdamin, \lambdamax]$,  we can write 
\begin{align*} 
G_2^{-c}  
& = \min_{t, v}\left \{ t \times \lambda \colon ~~~ -\frac{c-1}{c+1} \leq t \leq 1, ~~~ \lambdamin \leq \lambda \leq  \lambdamax \right\} \\
& = \min\left (
\lambda_{\min}, ~~ 
-\frac{c-1}{c+1}  \lambda_{\max}  
\right).
\end{align*}
From $\lambda_{\min} \leq \lambda_{\max}$, we can easily see that $G_2^{-c}\leq 0$, and  hence the form above is equivalent to the result in  Theorem~\ref{thm:2spliting}. The corresponding optimal solutions follow Lemma~\ref{lem:solve2min} and \ref{lem:solve2max} below,
which describe the values of $(w_1,w_2, r_1, r_2)$ to minimize and maximize 
$w_1r_1^2 + w_2 r_2^2$, respectively. 
\end{proof}

\begin{lem}\label{lem:solve2min}
Consider the following optimization with $c \geq 1$: 
\begin{align}\label{equ:R2min}
\begin{split} 
    R_2^{\min}:= 
    \min_{(\vv w, \vvr) \in \RR^4} 
     w_1 r_1^2 & + w_2 r_2^2\\  ~~~~
    s.t. ~~~&  w_1 + w_2 = 1 \\
    & w_1 r_1 + w_2 r_2 = 0 \\
    & |w_1| + |w_2|  \leq c \\
    & |r_1|, |r_2| \leq 1. 
    \end{split}
    \end{align}
Then we have $R_2^{\min} = - \frac{c-1}{c+1}$ and the optimal solution is achieved by the following scheme: 
        \begin{align}
        \label{equ:neg2r}
        \begin{split} 
        & w_1 = - \frac{c-1}{2}, ~~~~ r_1 =1 \\  
        & w_2 = \frac{c+1}{2}, ~~~~ r_2 = \frac{c-1}{c+1}. 
        \end{split}
    \end{align}
\end{lem}
\begin{proof}
\textbf{Case 1 ($w_2 \leq 0$, $w_1 \geq 1$)}~~
Assume $w_2 = - a$. We have $w_1 = 1+a >1$. 
\begin{align*}
\min_{a, r_1, r_2} 
     (1+a) r_1^2 & - a r_2^2\\  ~~~~
    s.t. ~~~&  
     (1+a) r_1 = a r_2 \\
    & a   \leq \frac{c-1}{2} \\
    & |r_1|, |r_2| \leq 1. 
\end{align*}
Eliminating $r_1$, we have 
\begin{align*}
\min_{a,  r_2} 
     \left (-\frac{a}{1+a} \right ) r_2^2
     ~~~~~s.t.~~~~~
     a   \leq \frac{c-1}{2}, ~~~~~
      |r_2| \leq 1. 
\end{align*}
The optimal solution is $r_2 =  1$ or  $-1$, and $a = \frac{c-1}{2}$, for which we  achieve a minimum value of $w_1 r_1^{2} +w_2 r_2^{2} =- \frac{c-1}{c+1}$. 

\noindent \textbf{Case 2 ($w_1 \geq 0$, $w_2\geq 0$)}~~ 
This case is obviously sub-optimal since we have $w_1 r_1^2 + w_2 r_2^2 \geq 0 \geq - \frac{c-1}{c+1}$ in this case. 

Overall, the minimum value is $-\frac{c-1}{c+1}$. This completes the proof. 
\end{proof}

\begin{lem}\label{lem:solve2max}
Consider the following optimization with $c \geq 1$: 
\begin{align}
\label{equ:R2max}
\begin{split}
    R_2^{\max}:=  
    \max_{(\vv w, \vvr) \in \RR^4} 
     w_1 r_1^2 & + w_2 r_2^2\\  ~~~~
    s.t. ~~~&  w_1 + w_2 = 1 \\
    & w_1 r_1 + w_2 r_2 = 0 \\
    & |w_1| + |w_2|  \leq c \\
    & |r_1|, |r_2| \leq 1. 
    \end{split} 
    \end{align}
Then we have $R_2^{\max} = 1$, which is achieved by the following scheme: 
        \begin{align}\label{equ:p2r}\begin{split} 
        & w_1 =  \frac{1}{2}, ~~~~ r_1 =1 \\  
        & w_2 = \frac{1}{2}, ~~~~ r_2 =-1.
        \end{split}
    \end{align}
\end{lem}
\begin{proof}
It is easy to see  that $R_2^{\max}\leq 1$. On the other hand, this bound is achieved by the scheme in \eqref{equ:p2r}. 
\end{proof}

\newpage \clearpage

\subsection{Derivation of Triplet Splittings ($m=3$)}


\begin{thm}
 Consider the optimization in \eqref{equ:optIIneg} with $m =3$ and $c\geq 1$. 
 
1)  The optimal solution of  \eqref{equ:optIIneg} must satisfy 
$$\delta_i = \sum_{\ell=1}^{d_\lambda}r_{i,\ell} v_\ell, 
$$
where $\{v_\ell \colon i=1,\ldots,{d_\lambda}\}$  is a set of $d_\lambda$ orthonormal eigen-vectors of $S(\theta)$ that share the same eigenvalue $\lambda$, 
and $\{r_{i,\ell}\}_{i,\ell}$ is a set of coefficients. 

2) Write $r_i = [r_{i,1}, \ldots, r_{r,d_\lambda}]\in \RR^{d_\lambda}$ for $i=1,2,3$. 
The optimization in \eqref{equ:optIIneg} is equivalent to 
\begin{align} \label{equ:R3Opt}
\begin{split} 
    G_3^{-c}:= 
    \min_{\vv w, \vv r, \lambda} 
     (w_1 \norm{r_1}^2 & + w_2 \norm{r_2}^2 + w_3 \norm{r_3}^2)\times \lambda 
     \\
    s.t. ~~~&  w_1 + w_2 + w_3 = 1 \\
    & w_1 r_1 + w_2 r_2 + w_3 r_3  = 0 \\
    & |w_1| + |w_2|   + |w_3|\leq c \\
    & \norm{r_1}, \norm{r_2}, \norm{r_3} \leq 1 \\
    & \text{$\lambda$ is an eigen-value of $S(\theta)$ with $d_\lambda$ orthogonal eigen-vectors}. 
    \end{split}
    \end{align}
    
3)    
The optimal value above is 
\begin{align}\label{equ:G2capp}
G_{3}^{-c} = 
\min\left (  \frac{c+1}{2}\lambdamin, ~~  -\frac{c-1}{2} \lambdamax , ~~
0 \right ). 
\end{align}
If $-\frac{c-1}{2} \lambda_{\max} <  \frac{c+1}{2}\min(\lambdamin,0)$, the optimal solution is achieved by 
\begin{align*} 
        \left (w_1 = - \frac{c-1}{4}, ~~~~ \delta_1 =v_{\max}\right), &&  
         \left (w_2 = - \frac{c-1}{4}, ~~~~ \delta_2 =  - v_{\max} \right), &&
         \left (w_3 = \frac{c+1}{2}, ~~~~ \delta_3 = 0 \right).          
        \end{align*}
If $
\frac{c+1}{2}\lambda_{\min} <  
-\frac{c-1}{2}\max(\lambdamax,0)$ 

the optimal solution is achieved by 
\begin{align*} 
        \left (w_1 =  \frac{c+1}{4}, ~~~~ \delta_1 =v_{\min}\right), &&  
         \left (w_2 =  \frac{c+1}{4}, ~~~~ \delta_2 =  - v_{\min} \right), &&
         \left (w_3 = -\frac{c-1}{2}, ~~~~ \delta_3 = 0 \right).          
        \end{align*}
If $0\leq \min\left( \frac{c+1}{2}\lambdamin, ~-\frac{c-1}{2} \lambdamax \right)$, and hence 
$\lambdamin =\lambdamax = 0$, the optimal solution can be achieved by no splitting: $\delta_1=\delta_2=\delta_3=0$.         
\end{thm} 
\begin{proof}
 
1-2) Following Lemma~\ref{lem:eigen}, the optimal $\delta_1,\delta_2, \delta_3$ are eigen-vectors of $S(\theta)$. 
Because eigen-vectors associated with different eigen-values are linearly independent, we have that  
$\delta_1,\delta_2, \delta_3$ must share the same eigen-value (denoted by $\lambda$) due to the constraint  $w_1 \delta_1 + w_2 \delta_2 + w_3 \delta_3 =0$. 
Assume $\lambda$ is associated with $d_\lambda$ orthonormal eigen-vectors $\{v_\ell\}_{\ell=1}^{d_\lambda}$. 
Then we can write 
$\delta_i = \sum_{\ell} r_{i,\ell} v_\ell$ for $i=1,2,3$, for  which $\norm{\delta_i} = \norm{r_i}$ 
and $\delta_i^\top S(\theta) \delta_i = \lambda \norm{r_i}^2$. 
It is then easy to reduce \eqref{equ:optIIneg} to \eqref{equ:R3Opt}. 
3) Following Lemma~\ref{lem:3vecmax} and \ref{lem:3vecmin},   
the value of $w_1 \norm{r_1}^2 + w_2 \norm{r_2}^2 + w_3 \norm{r_3}^2$ in \eqref{equ:R3Opt} can range from $-\frac{c-1}{2}$ to $\frac{c+1}{2}$, for any positive integer $d_\lambda$. 
In addition, the range of the eigen-value $\lambda$  is $[\lambdamin, \lambdamax]$.  
Therefore, we can write 

\begin{align*}
G_3^{-c}  
& = \min_{t, v}\left \{ t \times \lambda \colon ~~~ -\frac{c-1}{2} \leq t \leq \frac{c+1}{2}, ~~~ \lambdamin \leq \lambda \leq  \lambdamax \right\} \\
& = \min\left(
-\frac{c-1}{2} \lambdamax, ~~~
\frac{c+1}{2} \lambdamin
\right).
\end{align*}

Because $\lambda_{\min} \leq \lambda_{\max}$, we have $G_3^{-c}\leq 0$ and hence the result above is equivalent to the form in \eqref{equ:G2capp}. The corresponding optimal solutions follow Lemma~\ref{lem:3vecmax} and \ref{lem:3vecmin}. 

\end{proof}

\begin{lem} \label{lem:3vecmax}
For any $c\geq 1$ and 
any positive integer $d_{r}$, define 
\begin{align}\label{equ:3splitingmaxvec}
\begin{split} 
    R_{3,c,d_\lambda}^{\max} = \max_{\vv w \in \RR^3, \vvr\in \RR^{3\times d_\lambda}} 
     (
     w_1 \norm{\vvr_{1}}^2 & + w_2 \norm{\vvr_{2}}^2 + w_3 \norm{\vvr_{3}}^2)\\   
    s.t. ~~~&  w_1 + w_2 + w_3 = 1 \\
        & |w_1| + |w_2|   + |w_3|\leq c \\
    & w_1 \vvr_{1} + w_2 \vvr_{2} + w_3 \vvr_{3}  = 0,~~~\forall \ell \\
    & \norm{\vvr_{i}} \leq  1 ~~~\forall i=1,2,3.
\end{split}
\end{align}   
Then we have $R_{3,c,d_\lambda}^{\max} = \frac{c+1}{2}$ and the optimum is achieved by     
\begin{align}\label{equ:3splitingmaxvecSolution} 
\left ( w_1 = \frac{c+1}{4},~~~
\vvr_1 = \vve
\right ) &&
\left ( w_2 = \frac{c+1}{4},~~~
\vvr_2 = -\vve
\right ) &&
\left ( w_3 = -\frac{c-1}{2},~~~
\vvr_3 = \vv 0
\right ),
\end{align}
where $\vve$ is any vector whose norm equals one, that is, $\norm{\vve} = 1$. 
\end{lem}
\begin{proof}
First, it is easy that verify that 
$R_{3,c,d_\lambda}^{\max} \geq \frac{c+1}{2}$ by taking the solution in \eqref{equ:3splitingmaxvecSolution}. 
We just need to show that $R_{3,c,d_\lambda}^{\max}\leq \frac{c+1}{2}$. 

Define $w_{i}^+ = \max(w_i, 0)  = ({w_i  + |w_i|})/{2}$. From $w_1+w_2+w_3 = 1$ and $|w_1|+|w_2| + |w_3| \leq c$, we have 
$$
w_1^+ +  w_2^+ + w_3^+ = 
\frac{(w_1 + |w_1| +  w_2 + |w_2| + w_3 + |w_3|)}{2} \leq \frac{c+1}{2}. 
$$
Therefore, under the constraints in \eqref{equ:3splitingmaxvec}, we have 
\begin{align*}
  R_{3,c,d_\lambda}^{\max}
  & = \max_{\vv w, \vvr}
     (
     w_1 \norm{\vvr_{1}}^2 + w_2 \norm{\vvr_{2}}^2 + w_3 \norm{\vvr_{3}}^2) \\
  & \leq \max_{\vv w, \vvr} (w_1^+ \norm{\vvr_{1}}^2 + w_2^+ \norm{\vvr_{2}}^2 + w_3^+ \norm{\vvr_{3}}^2) \\
     & \leq \max_{\vv w, \vvr} (w_1^+ + w_2 ^+ + w_3^+) \\
     & \leq \frac{c+1}{2}.  
\end{align*}
\end{proof}

\begin{lem}\label{lem:3vecmin} 
For any $c\geq 1$ and 
any positive integer $d_{r}$, define 
\begin{align}\label{equ:3splitingminvec}
\begin{split} 
    R_{3,c,d_\lambda}^{\min} = \min_{\vv w \in \RR^3, \vvr\in \RR^{3\times d_\lambda}} 
     (
     w_1 \norm{\vvr_{1}}^2 & + w_2 \norm{\vvr_{2}}^2 + w_3 \norm{\vvr_{3}}^2)\\    
    s.t. ~~~&  w_1 + w_2 + w_3 = 1 \\
        & |w_1| + |w_2|   + |w_3|\leq c \\
    & w_1 \vvr_{1} + w_2 \vvr_{2} + w_3 \vvr_{3}  = 0,~~~\forall \ell \\
    & \norm{\vvr_{i}} \leq  1 ~~~\forall i=1,2,3.
    \end{split}
    \end{align}   
Then we have $R_{3,c,d_\lambda}^{\min} = -\frac{c-1}{2}$ and the optimum is achieved by     
\begin{align}\label{equ:3splitingminvecSolution} 
\left ( w_1 = -\frac{c-1}{4},~~~
\vvr_1 = \vve
\right ) &&
\left ( w_2 = -\frac{c-1}{4},~~~
\vvr_2 = -\vve
\right ) &&
\left ( w_3 = \frac{c+1}{2},~~~
\vvr_3 = \vv 0
\right ),
\end{align}
where $\vve$ is any vector whose norm equals one, that is, $\norm{\vve} = 1$. 
\end{lem}
\begin{proof}
First, it is easy that verify that 
$R_{3,c,d_\lambda}^{\min} \leq -\frac{c-1}{2}$ by taking the solution in \eqref{equ:3splitingminvecSolution}. 
We just need to show that $R_{3,c,d_\lambda}^{\min}\geq -\frac{c-1}{2}$. 

Define $w_{i}^- = \min(w_i, 0)  = ({w_i  - |w_i|})/{2}$. From $w_1+w_2+w_3 = 1$ and $|w_1|+|w_2| + |w_3| \leq c$, we have 
$$
w_1^- +  w_2^- + w_3^- = 
\frac{(w_1 - |w_1| +  w_2 - |w_2| + w_3 - |w_3|)}{2} \geq -\frac{c-1}{2}. 
$$
Therefore, under the constraints in \eqref{equ:3splitingminvec}, we have 
\begin{align*}
  R_{3,c,d_\lambda}^{\min}
  & = \min_{\vv w, \vvr}
     (
     w_1 \norm{\vvr_{1}}^2 + w_2 \norm{\vvr_{2}}^2 + w_3 \norm{\vvr_{3}}^2) \\
  & \geq \min_{\vv w, \vvr} (w_1^- \norm{\vvr_{1}}^2 + w_2^- \norm{\vvr_{2}}^2 + w_3^- \norm{\vvr_{3}}^2) \\
     & \geq \min_{\vv w, \vvr} (w_1^- + w_2 ^- + w_3^-) \\
     & \geq -\frac{c-1}{2}.  
\end{align*}
\end{proof}


\subsection{Derivation of the Optimal Quartet Splitting $(m=4)$}

\begin{thm}
Let $\lambda_{\min}$, $\lambda_{\max}$ be the  smallest and largest eigenvalues of $S(\theta)$, respectively, and $\lambda_{\min}$, $\lambda_{\max}$ their corresponding eigen-vectors. 
For the optimization in \eqref{equ:optIIneg}, we have for any positive integer $m$ and $c\geq 1$, 
$$
G_m^{-c} \geq\frac{c+1}{2}\min(\lambda_{\min},0)+\frac{1-c}{2}
\max(\lambda_{\max},0). 
$$
In addition, this lower bound is achieved by splitting the neuron to $m=4$ copies, with 
\begin{align}\label{equ:4splittingappendix}
\begin{split} 
\w_{1}=\w_{2}=\frac{c+1}{4}, & \ \ \ \ \ \  \w_{3}=\w_{4}=\frac{1-c}{4}\\
\delta_{1}=-\delta_{2}=
\ind(\lambdamin\leq 0)
v_{\min}, & \ \ \ \ \ \  \delta_{3}=-\delta_{4}=
\ind(\lambdamax\geq 0)v_{\max},
\end{split}
\end{align}
where $\ind(\cdot)$ denotes the indicator function. 
\end{thm}
\begin{proof}
Denote by $I_{\vv w}^+ :=\{i \in [m]\colon  w_i > 0\}$ and $I_{\vv w}^- :=\{i \in [m]\colon  w_i < 0\}$ the index set of positive and negative weights, respectively. 
And $S_{\vv w}^+ =  \sum_{i\in I_{\vv w}^+} w_i$ the sum of the positive weights.  We have $\sum_{i} |w_i| = 2 S_{\vv w}^+ -1 \leq c$,
yielding  $0\leq S_{\vv w}^+  \leq  (c+1)/2$. 

Note that we have $\delta_i^\top S(\theta) \delta_i \in [\min(\lambdamin, 0), ~ \max(\lambdamax, 0)]$ for $\norm{\delta_i} \leq 1$. we have  
\begin{align*} 
G_m^{-c} 
& = \min \left \{ \sum_{i\in I_{\vv w}^+} w_i \delta_i^\top S(\theta) \delta_i + \sum_{i\in I_{ \vv w}^-}  w_i \delta_i^\top S(\theta) \delta_i \right\} \\
& \geq S_{\vv w}^+ \min(\lambdamin,0) ~+~ 
(1-S_{\vv w}^+) \max(\lambdamax,0) \\
& \geq \frac{c+1}{2}\min(\lambdamin,0) 
+ \frac{1-c}{2} \max(\lambdamax, 0). 
\end{align*}
On the other hand, it is easy to verify that this bound is achieved by the solution in \eqref{equ:4splittingappendix}. This completes the proof. 
\end{proof}

\section{Convergence Analysis}
We provide a simple analysis of the convergence of the training loss of signed splitting steepest descent (Algorithm~\ref{alg:main}) on one-hidden-layer neural networks. 
We show that our algorithm allows us to achieve a training MSE loss of $\eta$ by splitting at most $\mathcal O((n/(d\eta))^{3/2})$ steps, starting from a single-neuron network,
where $n$ is data size and $d$ the dimension of the input dimension. 

To set up, consider splitting a one-hidden-layer network, 
\begin{align}\label{equ:onelayer}
f(x;~\vv\theta,\vv w)=\sum_{i=1}^{m}w_{i}\sigma(\theta_{i}^\top x), 
\end{align}
where $\sigma\colon \RR\to\RR$ is an uni-variate activation function. 
Each of the $m$ neurons 
can be the offspring of some earlier neuron, 
and will be split further. 
Consider a general loss of form  $$L(\vv\theta,\vv w) = \E_{x\sim \Dn}[\Phi(f(x;\vv\theta,\vv w))].$$ 
The splitting matrix of the $i$-th neuron can be shown to be 
\[
\S_i(\vv\theta,\vv w)=
w_i\E_{x\sim \Dn}\left [ 
\Phi'(f(x; ~\vv\theta, \vv w)) 
\sigma''(\theta_i^{\top} x) x x^{\top} \right ]. 
\]
For an empirical dataset $\Dn = 
\{x^{(\ell)}\}_{\ell=1}^n$, define 
\[
\X=\left[\Vec(x^{(1)}x^{(1)\top}),...,\Vec(x^{(n)}x^{(n)\top})\right]\in\R^{d^{2}\times n}.
\]
where $\Vec(A)$ denotes the vectorization of matrix $A$. 

We start with showing that 
the training loss can be controlled by the spectrum radius $\rho(S_i(\vv\theta,\vv w))$ of the splitting matrix of any neuron.  
This allows us to establish provable bounds on the loss because 
$\rho(S_i(\vv\theta,\vv w))$ is expected to be zero or small 
when the signed splitting descent converges. 

\begin{ass} \label{ass:ass1}
Consider the network in \eqref{equ:onelayer} with mean square loss $\Phi(f(x)) := {\color{black}\frac{1}{2}}(f(x) - y(x))^2$, where $y(x)$ denotes the label associated with $x$. 
Assume 
$\lambda_{\X}:=\lambda_{\min}\left(\X^{\top}\X/d^{{2}}\right) > 0$, 
and $|\sigma''(\theta_i^\top x^{(\ell)})|\geq h$ 
for $i\in [m]$ and $\ell \in [n]$. 
Assume $ \norm{\nabla^3_{\vv\theta^3} L(\vv\theta, \vv w)}_{\infty} \leq 6C$ for all the values of $\vv\theta$ and $\vv w$ reachable by our algorithm. 
\end{ass}

\paragraph{Remark}
Notice that the assumption $\lambda_\X>0$ requires that $d^2 > 0$, which holds for most computer vision dataset, e.g. , CIFAR and ImageNet.


\begin{lem} \label{lem:globalopt} Under Assumption~\ref{ass:ass1}, denote by  
$\rho({S_i}):=\max\left\{\left|\lambda_{\max}\left(\S_i(\vv\theta,\vv w)\right)\right|,\left|\lambda_{\min}\left(\S_i(\vv\theta,\vv w)\right)\right|\right\}$ the spectrum radius of $\S_i(\vv\theta,\vv w)$, and 
 $\alpha = {n}/({d h^2 \lambda_\X})$.  
 We have 
\[
\E_{x\sim \Dn}\left[(f(x;\vv\theta,\vv w)-y(x))^{2} \right] 
\le
\alpha {(\rho(S_i)/w_i)^{2}}, ~~\forall i \in [m]. 
\]
\end{lem}


\begin{ass} \label{ass:ass2}
Assume Assumption~\ref{ass:ass1} holds. 
Let $\eta$ be any positive constant, and define 
$\rho_0 :=  (\eta/\alpha)^{1/2}  = h  (\lambda_\X\eta d/n)^{1/2}$. 
Assume we apply Algorithm~\ref{alg:main} to the neural network in \eqref{equ:onelayer}, 
following the guidance below:

1) At each splitting step, we pick any neuron with $w_i^2 \geq 1$ and $\rho(S_i) \geq \rho_0$ 
and split it with the optimal splittings in  Theorem~\ref{thm:2spliting}-\ref{thm:4spliting} (with $m=2,3 \text{~or~}4$). The algorithm stops when either $L(\vv\theta,\vv w)\leq \eta$, or such neurons can not be found. 

2) Assume the step-size $\epsilon$ used in the splitting updates satisfies $\epsilon \leq\frac{1}{4C}\kappa_m  \rho_0 = \mathcal O((d\eta/n)^{1/2})$.  

3) Assume we only update $\vv\theta$ during the 
parametric optimization phase while keeping $\vv w$ unchanged, and the parametric optimization does not deteriorate the loss.  
\end{ass} 

\begin{thm} \label{thm:conv34splitting}
Assume we run Algorithm~\ref{alg:main} with triplet or quartet splittings $(m=3\text{~or~}4)$ and $c\geq 3$, and 
Assumption~
\ref{ass:ass2} holds. 
If we initialize the network 
with a single neuron that satisfies $w_{i}^2\geq 1$, then the algorithm 
achieves $L(\vv\theta, \vv w)\leq \eta$ 
with at most 
$
T:=\left \lceil  \beta  \epsilon^{-2}  
\left (\frac{n}{d\eta}\right)^{1/2} 
\right \rceil 
$
iterations, 
where 
\[
\beta = 4(\kappa_3 h \sqrt{\lambda_\X} )^{-1}
\max(L(\vv\theta_0, \vv w_0) - \eta,0).
\]
In this case, we are able to obtain a neural network that achieves $L(\vv \theta, \vv w)\leq \eta$ with  $2T+1$ neurons
via triplet splitting, 
and $3T+1$ neurons via quartet splitting. 
\end{thm}

Since $\epsilon=\mathcal O((d\eta/n)^{1/2})$ by Assumption \ref{ass:ass2}, 
our result 
suggests that we can 
learn a neural network with $\mathcal O((n/(d\eta))^{3/2})$ neurons 
to achieve a loss  no larger than $\eta$. 
Note that this is much smaller than the number of neurons required 
for over-parameterization-based analysis of standard gradient descent training of neural networks. For example, the analysis in \citet[][]{du2018gradient} requires   $\mathcal O(n^6)$ neurons, or $\mathcal O(n^2/d)$ in \citet[][]{oymak2019towards},  
much larger than what we need when $n$ is large.

Similar result can be established for binary splittings ($m=2$), 
but extra consideration is needed. The problem is that the positive binary splitting halves the output weight of the neurons at each splitting, 
which makes $w_i$ of all the neurons small and hence yields a loose bound in \eqref{lem:globalopt}. This problem  is  sidestepped in Theorem~\ref{thm:conv34splitting}  
for triplet and quartet splittings by taking $c\geq 3$ to ensure $(c+1)/4\geq1$, 
so that  there always exists at least one off-spring  
whose output weight is larger than 1 after the splitting. 
There are several different ways for addressing this issue for binary splittings.  
One simple approach is to initialize the  network to have $T$ neurons with $w_i^2\geq1$, so that there always exists neurons with $w_i^2\geq 1$ during the first $T$ iterations, and yields a neural network with $2T$ neurons at end. We provide a throughout discussion of this issue in the Appendix.  

\section{Proof of Theoretical Analysis} \label{apx: theory}

\subsection{Proof of Lemma \ref{lem:globalopt}}

\noindent\textbf{Lemma~\ref{lem:globalopt}}~~ 
\emph{Under Assumption~\ref{ass:ass1}, denote by  
$\rho({S_i}):=\max\left\{\left|\lambda_{\max}\left(\S_i(\vv\theta,\vv w)\right)\right|,\left|\lambda_{\min}\left(\S_i(\vv\theta,\vv w)\right)\right|\right\}$ the spectrum radius of $\S_i(\vv\theta,\vv w)$, and 
 $\alpha = {n}/({d h^2 \lambda_\X})$.  
 We have 
\[\E_{x\sim \Dn}\left[(f(x;\vv\theta,\vv w)-y(x))^{2} \right] 
\le
\alpha {(\rho(S_i)/w_i)^{2}}, ~~\forall i \in [m]. \]
}
\begin{proof}
We want to bound the mean square error 
using the spectrum radius of the splitting matrix. For the mean square loss, a derivation shows that the splitting matrix of the $i$-th neuron is 
\begin{align*} 
S_i = 
\frac{1}{n}\sum_{\ell=1}^n w_i e_\ell  h_{i,\ell} \left (   x^{(\ell)} x^{(\ell)\top} \right),&& 
e_\ell := f(x^{(\ell)};~ \vv\theta,\vv w) - y(x^{(\ell)}),&&
h_{i,\ell}:=\sigma''(\theta_i^\top x^{(\ell)}). 
\end{align*}
Denote $\norm{\cdot}_F$ as the Frobenius norm. We have 
\begin{align*}
  \norm{S_i}_F = \norm{\Vec(S_i)}_2 
 & = \frac{1}{n}  \norm{\sum_{\ell=1}^n w_i e_\ell  h_{i,\ell} \Vec \left (   x^{(\ell)} x^{(\ell)^\top} \right) }_2\\ 
& \geq d \sqrt{\lambda_{\X}} \frac{1}{n} \sqrt{\sum_{\ell=1}^n (w_i e_\ell h_{i,\ell})^2 }  \\
& \geq d \sqrt{\lambda_\X} \frac{1}{n} |w_i| h \sqrt{\sum_{\ell=1}^n ( e_\ell)^2 } , \ant{because $|h_{i,\ell}|\geq h$, $\forall i,\ell$}. 
\end{align*}
On the other hand, 
\begin{align*}
     \norm{S_i}_F^2 = \trace(S_i^2)  
     \leq d \rho(S_i^2)  =
      d \rho(S_i)^2,
\end{align*}
where $ \rho(S_i^2) = \rho(S_i)^2$ holds because $S_i$ is a symmetric matrix.  
This gives 
$$
\E_{x\sim  D_n}
\left [(f(x; \vv\theta,\vv w)-y(x))^2\right]
= \frac{1}{n}{\sum_{\ell=1}^n ( e_\ell)^2 } \leq  
\frac{n  \rho(S_i)^2}{ h^2 d\lambda_\X w_i^2}. 
$$
\end{proof}

\subsection{Convergence Rate of Triplet and Quartet Splittings} 

\begin{proof}[\textbf{Proof of Theorem~\ref{thm:conv34splitting}}] 
First, 
when $c\geq 3$, 
note that the triplet and quartet splittings always yield at least one off-spring whose weight's absolute value is no smaller than 1 (because $(c+1)/2\geq (c+1)/4 \geq 1$ when $c\geq 3$). Since there is at least one neuron satisfies $w_{i}^2\geq 1$ in the initialization, there always exist neurons with $w_i^2 \geq 1$ throughout the algorithm. 

If there exists a neuron $i$ such that $w_i^2 \geq 1$ and $\rho(S_i) \leq |w_i| (\eta/\alpha)^{1/2}$  
within the first $T$ iterations of signed splitting,
 we readily have by Lemma~\ref{lem:globalopt}
$$
L(\vv\theta, \vv w) \leq \alpha (\rho(S_i)/w_i)^2 \leq \eta. 
$$
If this does not hold, then 
$\alpha(\rho(S_i)/w_i)^2 \geq \eta$ holds for every neuron in the first $T$ iterations.  
This means that the neurons with $w_i^2\geq 1$ must have $\rho(S_i) \geq (\eta/\alpha)^{1/2}$. 

Let $(\vv\theta',\vv w')$ be the parameter and weights we obtained by applying an optimal triplet splitting on $(\vv\theta, \vv w)$. We have by the Taylor expansion in Theorem~2.2 and Theorem 2.4 of \citet{splitting2019}, we have 
\begin{align*} 
L(\vv\theta',\vv w') 
& \leq L(\vv \theta, \vv w) + 
 \frac{\epsilon^2}{2} G_m^{-c} +  C\epsilon^3 \\
& \leq L(\vv \theta, \vv w) -\frac{\epsilon^2}{2} \kappa_3 \rho(S_i) +  C\epsilon^3 \\
& \leq L(\vv \theta, \vv w) -\frac{\epsilon^2}{2} \kappa_3 ({\eta /\alpha })^{1/2}
+  C\epsilon^3. 
\end{align*}
Therefore, through the first $T$ iterations of  splitting descent,  we have 
\begin{align*} 
L(\th_T, \vv w_T) 
& \leq L(\vv \theta_0, \vv w_0) - T \left (\frac{\epsilon^2}{2} \kappa_3 
({\eta /\alpha } )^{1/2}
- C \epsilon^3 \right) \\
& \leq L(\vv \theta_0, \vv w_0) - T \left (\frac{\epsilon^2}{4} \kappa_3 
({\eta /\alpha } )^{1/2}\right) \ant{because we assume $\epsilon \leq\frac{1}{4C}\kappa_3  (\eta/\alpha)^{1/2}$} \\
& \leq \eta. 
\ant{because we assume  $T=\left \lceil 4\left (\epsilon^{2}\kappa_3 (\eta/
\alpha)^{1/2}\right)^{-1} (L(\vv\theta_0, \vv w_0) - \eta) \right \rceil$}
\end{align*}
This completes the proof. 
\end{proof}
\subsection{Convergence of Binary Splitting}


\begin{thm} \label{thm:conv2splitting}
Assume we run Algorithm~\ref{alg:main} with signed binary splittings $(m=2)$  and $c\geq 1$, and  
Assumption~
\ref{ass:ass2} holds.  
Assume we initialize the network with $(\vv\theta_0,\vv w_0)$ with $m_0$ neurons such that there is at least 
\[
T:=\left \lceil \beta 
n^{1/2} d^{-1/2} 
\epsilon^{-2} \eta^{-1/2}  \right \rceil 
\]
neurons satisfying $w_i^2 \geq 1$, 
where $\beta  
=4(\kappa_2 h^2 \lambda_\X^2 )^{-1}  
\max(L(\vv\theta_0, \vv w_0) - \eta,0)$. 

Then the algorithm determines within at most 
$T$ 
iterations, and return 
a neural network that achieves $L(\vv \theta, \vv w)\leq \eta$ with  $T+m_0$  neurons.  
\end{thm}

As shown in Lemma~\ref{lem:initial2splitting}, we require the initialization condition of Lemma~\ref{lem:initial2splitting} in Theorem~\ref{thm:conv2splitting} with $m_0 = \mathcal O(n^{3/2} d^{-3/2} \eta^{-3/2})$, 
which implies that signed binary splitting can learn neural networks with 
$ \mathcal O(n^{3/2} d^{-3/2} \eta^{-3/2})$ neurons to achieve $L(\vv\theta,\vv w) \leq \eta$. 

\paragraph{Remark}
Indeed, if the initialization condition in Lemma~\ref{lem:initial2splitting} holds, Theorem \ref{thm:conv2splitting} can be generalized for triplet and quartet splitting easily for any $c>1$.



\begin{proof}[\textbf{Proof of Theorem~\ref{thm:conv2splitting}}]  
First, 
because there are at least $T$ neurons with $w_i^2 \geq 1$ and each splitting step splits only one such neurons, there must exist neurons with $w_i^2\geq 1$ throughout the first $T$ iterations. 

If there exists a neuron $i$ such that $w_i^2 \geq 1$ and $\rho(S_i) \leq |w_i| (\eta/\alpha)^{1/2}$  
within the first $T$ iterations of signed splitting,
 we readily have by Lemma~\ref{lem:globalopt}
$$
L(\vv\theta, \vv w) \leq \alpha (\rho(S_i)/w_i)^2 \leq \eta. 
$$
If this does not hold, then 
$\alpha(\rho(S_i)/w_i)^2 \geq \eta$ holds for every neuron in the first $T$ iterations.  
This means that the neurons with $w_i^2\geq 1$ must have $\rho(S_i) \geq (\eta/\alpha)^{1/2}$. 

Let $(\vv\theta',\vv w')$ be the parameter and weights we obtained by applying an optimal trplet splitting on $(\vv\theta, \vv w)$. By the Taylor expansion in Theorem~2.2 and Theorem 2.4 of \citet{splitting2019}, we have 
\begin{align*} 
L(\vv\theta',\vv w') 
& \leq L(\vv \theta, \vv w) + 
 \frac{\epsilon^2}{2} G_2^{-c} +  C\epsilon^3 \\
& \leq L(\vv \theta, \vv w) -\frac{\epsilon^2}{2} \kappa_2 \rho(S_i) +  C\epsilon^3 \\
& \leq L(\vv \theta, \vv w) -\frac{\epsilon^2}{2} \kappa_2 ({\eta /\alpha })^{1/2}
+  C\epsilon^3. 
\end{align*}
Therefore, through the first $T$ iterations of  splitting descent,  we have 
\begin{align*} 
L(\th_T, \vv w_T) 
& \leq L(\vv \theta_0, \vv w_0) - T \left (\frac{\epsilon^2}{2} \kappa_2  
({\eta /\alpha } )^{1/2}
- C \epsilon^3 \right) \\
& \leq L(\vv \theta_0, \vv w_0) - T \left (\frac{\epsilon^2}{4} \kappa_2 
({\eta /\alpha } )^{1/2}\right) \ant{because we assume $\epsilon \leq\frac{1}{4C}\kappa_3  (\eta/\alpha)^{1/2}$} \\
& \leq \eta. 
\ant{because we assume  $T=\left \lceil 4\left (\epsilon^{2}\kappa_3 (\eta/
\alpha)^{1/2}\right)^{-1} (L(\vv\theta_0, \vv w_0) - \eta) \right \rceil$}
\end{align*}
This completes the proof. 
\end{proof}

\begin{lem}\label{lem:initial2splitting}
Assume $\sigma(0) =0$, $\norm{\sigma}_{\mathrm{Lip}}< \infty$, and  
 $r_{\mathcal D_n} = \max_{(x,y)
\sim \mathcal D_n}(\norm{x}, |y|) < \infty$. 
Let $r$ be any positive constant and $\ell_0 = r_{\mathcal D_n} (1+ r \norm{\sigma}_{\mathrm{Lip}})$. 

If we initialize $(\vv \theta_0, \vv w)$ with $m_0 := \left \lceil \left (4\epsilon^{2}\kappa_2 (\eta/
\alpha)^{1/2}\right)^{-1} (\ell_0- \eta) \right \rceil$ neurons, 
such that $w_{i,0}^2 = 1$ and $\norm{\theta_i} \leq \frac{r}{m_0}$, 
then there exists at least 
$$
T = m_0 \geq 
\left \lceil \left (4\epsilon^{2}\kappa_2 (\eta/
\alpha)^{1/2}\right)^{-1} (L(\vv\theta_0, \vv w_0)- \eta) \right \rceil 
$$
neurons that satisfies $w_{i,0}^2 \geq 1$. 
\end{lem}

\begin{proof}
\begin{align*}
    L(\vv \theta_0,\vv w_0) 
    & = \E_{(x,y)\sim \mathcal D_n} \left [ \left ( 
    y - \sum_{i=1}^{m_0} w_{i,0}\sigma(\theta_{i,0}^\top x) \right)^2\right] \\
    & \leq  \E_{(x,y)\sim \mathcal D_n} \left [\left (|y| + m_0 \norm{\sigma}_{\mathrm{Lip}}  \norm{\theta} \norm{x}\right)^2\right ] \\
    & \leq r_{\mathcal D_n}(1 + 
     k \norm{\sigma}_{\mathrm{Lip}}) \\
     & = \ell_0.
\end{align*}
This obviously concludes the result. 
\end{proof}
\section{Experiment Detail}
\subsection{Toy RBF neural network}
\label{app:toy}
Following \citet{splitting2019}, we consider the following one-hidden-layer RBF neural network with one-dimensional inputs: 
\begin{align} \label{equ:rbf}
\!\!\!\!\!\!\! f(x) = 
\sum_{i=1}^{m} w_i \sigma(\theta_{i1} x + \theta_{i2}), && \sigma(x) = \exp (-\frac{x^2}{2}), 
\end{align}
where 
$x, w_i\in \R$ and $\theta_i=[\theta_{i1},\theta_{i2}]\in \RR^2$ for  $\forall i\in[m]$. We define an underlying true function by taking $m=15$ neurons and drawing $w_i$ and the elements of $\theta_i$ i.i.d. from  $\normal(0,3)$. 
We then simulate a dataset $
\{x^{(\ell)}, y^{(\ell)}\}_{\ell=1}^{1000}$ by 
drawing $x^{(\ell)}$ from $\mathrm{Uniform}([-5,5])$ and $y^{(\ell)} = f(x^{(\ell)})$. 

 Because different $m$ yields different numbers of new neurons at each splitting,
 we set the number of gradient descent iterations during the parametric update phase to be $(m-1)\times 10k$ for each $m\in\{2,3,4\}$,
 so that neural networks of the same size is trained with the same number of gradient descent iterations.    
 We use Adam optimizer 
  with learning rate 0.005 for parametric gradient descent in all the splitting methods. 
  
 \subsection{CIFAR-100}
 \label{app:cifar}
  We train MobileNetV1 starting from a small network with 32 filters in each layer. 
MobileNetV1 contains two type of layers: the depth-wise layer and the point-wise layer. 
We grow the network by splitting the filters in the point-wise layers following Algorithm \ref{alg:main}.  
In parametric update phase, we train the network for 160 epochs with batch size 128. We use stochastic gradient descent with momentum 0.1,  weight decay $10^{-4}$, 
{and learning rate $0.1$.}  
{We apply a decay rate 0.1 on the learning rate when we reach 50$\%$ and 75$\%$ of the total training epochs. } 
In the splitting phase, we set the splitting step size $\epsilon = 0.01$, choose $c = 1.3$, and split the top $35\%$ of the current filters at each splitting phase.  
Our experiments in the energy-aware setting follows 
 \citet{wang2019energy} closely, 
 by keeping their code unchanged expect replacing positive splitting schemes with binary signed splittings with $c=1.3$. 

\subsection{ImageNet}
\label{app:img}
We test different methods on MobileNetV1 following the same setting in \citet{wang2019energy}: when we update parameters, we set the batch size to be 128 for each GPU and use 4 GPUs in total, we use stochastic gradient descent with the cosine learning rate scheduler and set the initial learning rate to be 0.2. We also use label-smoothing (0.1) and 5 epochs warm-up. When splitting the networks, we set the splitting step size $\epsilon = 0.01$ and keep the number of MACs close to the MACs of S2D reported in \citet{wang2019energy}.

\subsection{Filter Visualization Result}
\label{sec:vis}
We visualize the filters picked by S2D/S3D 
when splitting VGG19. 
We take the pre-trained VGG19 on ImageNet, 
and evaluate the splitting matrices of the filters in the last  convolution layer on a randomly picked image. 
We then pick the filters with the largest or smallest 
$\lambda_{\min}$ or $\lambda_{\max}$ 
and visualize them before and after splitting. 
We use guided gradient back-propagation \citep{springenberg2014striving} as the visualization tool  
and apply a gray-scale on the output image to have a better visualization. During the splitting, we set the splitting step size $\epsilon = 0.01$. 

Figure \ref{fig:vis} visualizes the filters on an image whose label is ``bulldog'' (see Appendix~\ref{fig:appendix_vis} for more examples). We see that the filters with large    $\rho:=\max(|\lambda_{\min}|,|\lambda_{\max}|)$ tend to  change  significantly after splitting. In contrast, the filters with small $\rho$ tend to keep unchanged after splitting.  
This suggests that the spectrum radius $\rho$ provides a good estimation of the benefit of splitting. 
We further include three more visualization of the splitting of filters on ImageNet in Figure~\ref{fig:appendix_vis}. All the cases show similar consistent patterns illustrated above. 

\begin{figure*}[ht!]
    \centering
    \begin{overpic}[width=1\textwidth]{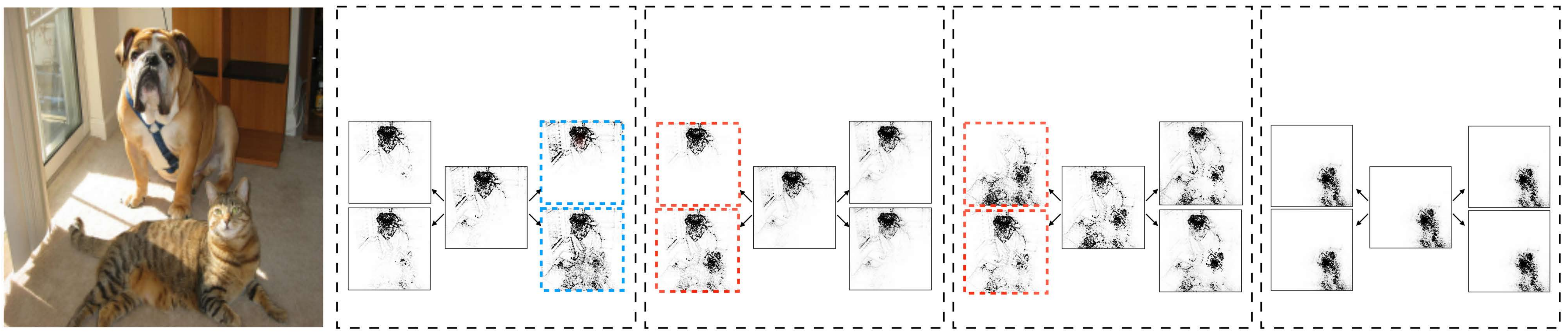}
    \put(22, 19){\scriptsize$\theta \pm \epsilon v_{\min}$}
    \put(22, 16){\scriptsize$\lambda_{\min}\text{=-}0.5$}
    \put(33.5, 19){\scriptsize$\theta \pm \epsilon v_{\max}$}
    \put(33, 16){\scriptsize$\lambda_{\max}\text{=}22.1$}
    
    \put(41.7, 19){\scriptsize$\theta \pm \epsilon v_{\min}$}
    \put(41.7, 16){\tiny$\lambda_{\min}\text{=-}29.2$}
    \put(53.2, 19){\scriptsize$\theta \pm \epsilon v_{\max}$}
    \put(53.2, 16){\tiny$\lambda_{\max}\text{=}10^{\text{-7}}$}
    
    \put(61.3, 19){\scriptsize$\theta \pm \epsilon v_{\min}$}
    \put(61.3, 16){\scriptsize$\lambda_{\min}\text{=-}23.6$}
    \put(72.8, 19){\scriptsize$\theta \pm \epsilon v_{\max}$}
    \put(72.8, 16){\scriptsize$\lambda_{\max}\text{=}0.3$}
    
    \put(80.9, 19){\scriptsize$\theta \pm \epsilon v_{\min}$}
    \put(80.9, 16){\scriptsize$\lambda_{\min}\text{=-}10^{\text{-}8}$}
    \put(92.3, 19){\scriptsize$\theta \pm \epsilon v_{\max}$}
    \put(92, 16){\scriptsize$\lambda_{\max}\text{=}10^{\text{-}2}$}
    
    \put(5,-1.5){\scriptsize(a) Original Image}
    \put(25,-1.5){\scriptsize(b) Largest $\lambda_{\max}$}
    \put(45,-1.5){\scriptsize(c) Smallest $\lambda_{\min}$}
    \put(65,-1.5){\scriptsize(d) Smallest $G_4^{-c}$}
    \put(83,-1.5){\scriptsize(e) Smallest $\rho (S(\theta))$}
    \end{overpic}
    \caption{{Filter visualization result. The filters are visualized by guided backpropagation. The center figures in (b)-(d) are the visualization of the different filters before splitting. The four surrounding figures in each panel 
 are the results we get after split the neurons with $\theta\gets \theta\pm \epsilon v_{\min}$ (left) and 
    $\theta\gets \theta\pm \epsilon v_{\max}$ (right), respectively. 
    The figures with colored dashed boxes 
    are the cases  with  significant changes after splitting,  whose corresponding eigen-values tend to be far away from zero comparing with the other cases.} 
    }
    \label{fig:vis}
\vspace{-8pt}
\end{figure*}

\begin{figure*}[h]
\centering
\scalebox{0.9}{
\setlength{\tabcolsep}{1.0pt}
\begin{tabular}{c}
\hspace{0.5em}\begin{overpic}[width=0.99\textwidth]{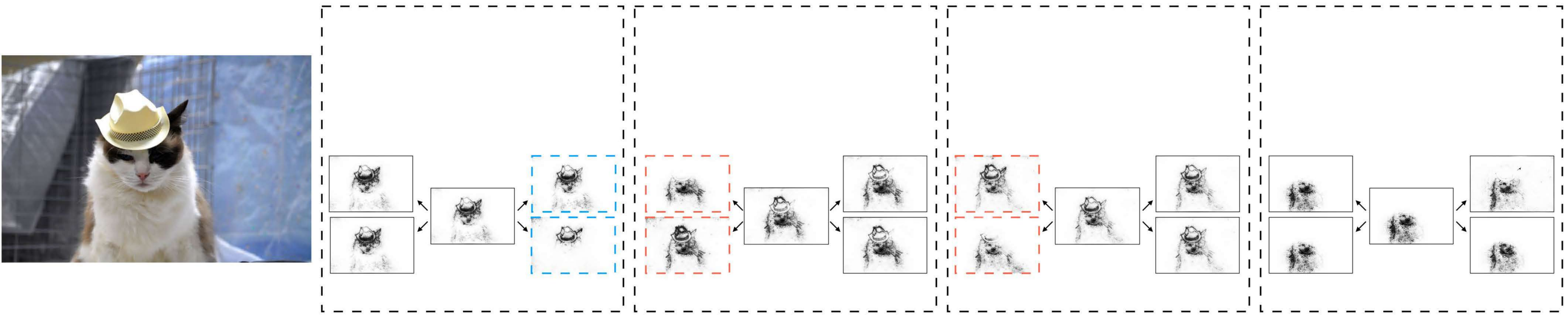}
\put(21, 17.3){\scriptsize$\theta \pm \epsilon v_{\min}$}
    \put(21, 14.3){\scriptsize$\lambda_{\min}\text{=-}10^{\text{-7}}$}
    \put(32.5, 17.3){\scriptsize$\theta \pm \epsilon v_{\max}$}
    \put(32, 14.3){\scriptsize$\lambda_{\max}\text{=}17.4$}
    
    \put(41.1, 17.3){\scriptsize$\theta \pm \epsilon v_{\min}$}
    \put(41.1, 14.3){\tiny$\lambda_{\min}\text{=-}44.7$}
    \put(52.6, 17.3){\scriptsize$\theta \pm \epsilon v_{\max}$}
    \put(52.6, 14.3){\tiny$\lambda_{\max}\text{=}10^{\text{-7}}$}
    
    \put(61.0, 17.3){\scriptsize$\theta \pm \epsilon v_{\min}$}
    \put(61.0, 14.3){\scriptsize$\lambda_{\min}\text{=-}23.5$}
    \put(72.5, 17.3){\scriptsize$\theta \pm \epsilon v_{\max}$}
    \put(72.5, 14.3){\scriptsize$\lambda_{\max}\text{=}10^{\text{-7}}$}
    
    \put(80.9, 17.3){\scriptsize$\theta \pm \epsilon v_{\min}$}
    \put(80.9, 14.3){\scriptsize$\lambda_{\min}\text{=-}10^{\text{-}8}$}
    \put(92.3, 17.3){\scriptsize$\theta \pm \epsilon v_{\max}$}
    \put(92, 14.3){\scriptsize$\lambda_{\max}\text{=}0.3$}

\end{overpic}\\
\begin{overpic}[width=1\textwidth]{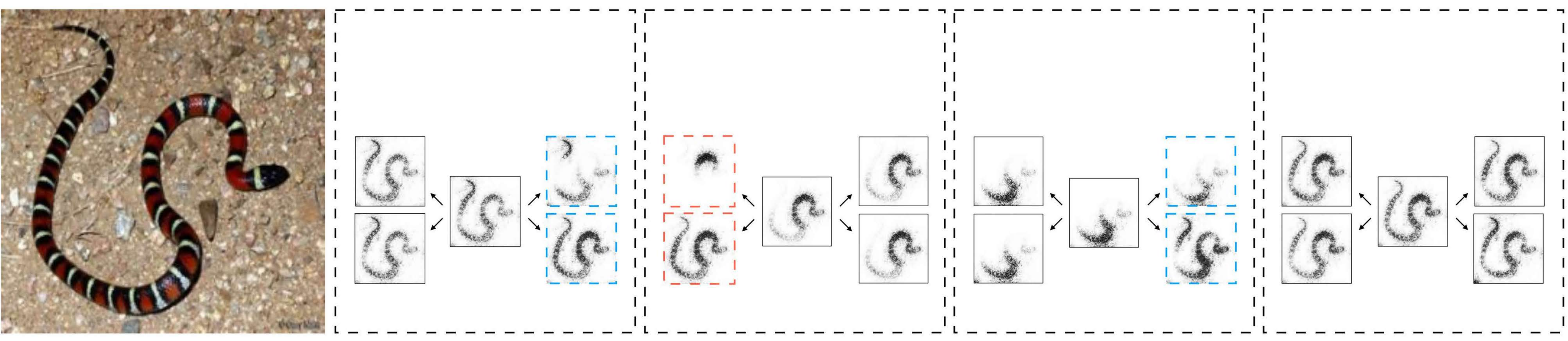}
\put(22, 19){\scriptsize$\theta \pm \epsilon v_{\min}$}
    \put(22, 16){\scriptsize$\lambda_{\min}\text{=-}10^{\text{-2}}$}
    \put(33.5, 19){\scriptsize$\theta \pm \epsilon v_{\max}$}
    \put(33, 16){\scriptsize$\lambda_{\max}\text{=}12.5$}
    
    \put(41.7, 19){\scriptsize$\theta \pm \epsilon v_{\min}$}
    \put(41.7, 16){\tiny$\lambda_{\min}\text{=-}21.3$}
    \put(53.2, 19){\scriptsize$\theta \pm \epsilon v_{\max}$}
    \put(53.2, 16){\tiny$\lambda_{\max}\text{=}10^{\text{-7}}$}
    
    \put(61.3, 19){\scriptsize$\theta \pm \epsilon v_{\min}$}
    \put(61.3, 16){\scriptsize$\lambda_{\min}\text{=-}16.2$}
    \put(72.8, 19){\scriptsize$\theta \pm \epsilon v_{\max}$}
    \put(72.8, 16){\scriptsize$\lambda_{\max}\text{=}10^{\text{-7}}$}
    
    \put(80.9, 19){\scriptsize$\theta \pm \epsilon v_{\min}$}
    \put(80.9, 16){\scriptsize$\lambda_{\min}\text{=-}10^{\text{-}9}$}
    \put(92.3, 19){\scriptsize$\theta \pm \epsilon v_{\max}$}
    \put(92, 16){\scriptsize$\lambda_{\max}\text{=}10^{\text{-}4}$}

\end{overpic}\\
\hspace{-0.6em}\begin{overpic}[width=1.01\textwidth]{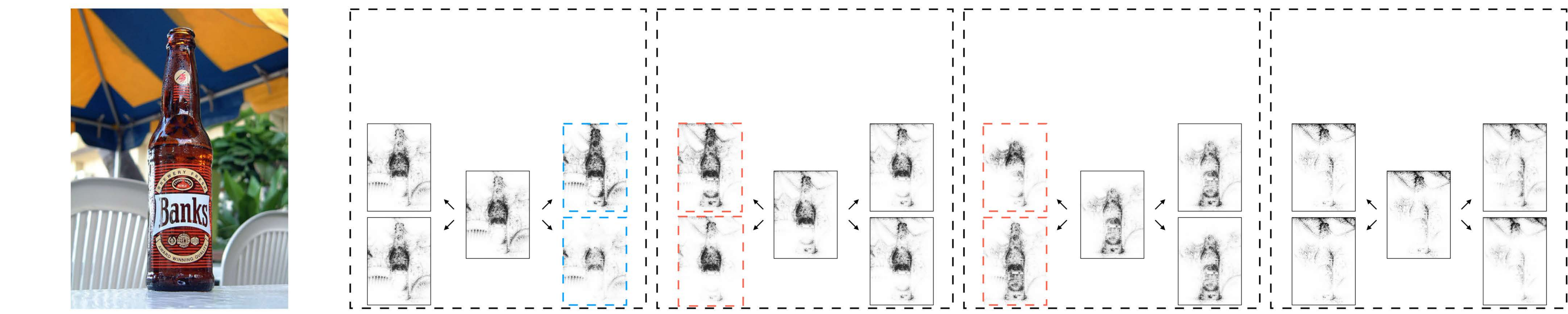}
\put(22.7, 17){\scriptsize$\theta \pm \epsilon v_{\min}$}
    \put(22.7, 14){\scriptsize$\lambda_{\min}\text{=-}10^{\text{-8}}$}
    \put(34.2, 17){\scriptsize$\theta \pm \epsilon v_{\max}$}
    \put(33.7, 14){\scriptsize$\lambda_{\max}\text{=}4.7$}
    
    \put(42.4, 17){\scriptsize$\theta \pm \epsilon v_{\min}$}
    \put(42.4, 14){\tiny$\lambda_{\min}\text{=-}7.5$}
    \put(53.9, 17){\scriptsize$\theta \pm \epsilon v_{\max}$}
    \put(53.9, 14){\tiny$\lambda_{\max}\text{=}10^{\text{-8}}$}
    
    \put(62, 17){\scriptsize$\theta \pm \epsilon v_{\min}$}
    \put(62, 14){\scriptsize$\lambda_{\min}\text{=-}5.6$}
    \put(73.3, 17){\scriptsize$\theta \pm \epsilon v_{\max}$}
    \put(73.3, 14){\scriptsize$\lambda_{\max}\text{=}10^{\text{-}5}$}
    
    \put(81.6, 17){\scriptsize$\theta \pm \epsilon v_{\min}$}
    \put(81.5, 14){\scriptsize$\lambda_{\min}\text{=-}10^{\text{-}9}$}
    \put(93, 17){\scriptsize$\theta \pm \epsilon v_{\max}$}
    \put(92.7, 14){\scriptsize$\lambda_{\max}\text{=}0.2$}
    
    \put(5,-1.5){\scriptsize(a) Original Image}
    \put(26,-1.5){\scriptsize(b) Largest $\lambda_{max}$}
    \put(46,-1.5){\scriptsize(c) Smallest $\lambda_{min}$}
    \put(66,-1.5){\scriptsize(d) Smallest $G_4^{-c}$}
    \put(84,-1.5){\scriptsize(e) Smallest $\rho (S(\theta))$}
\end{overpic}\\

\end{tabular}
}
\caption{\small More visualization results similar to Figure~\ref{fig:vis}. 
}
\label{fig:appendix_vis}
\end{figure*}
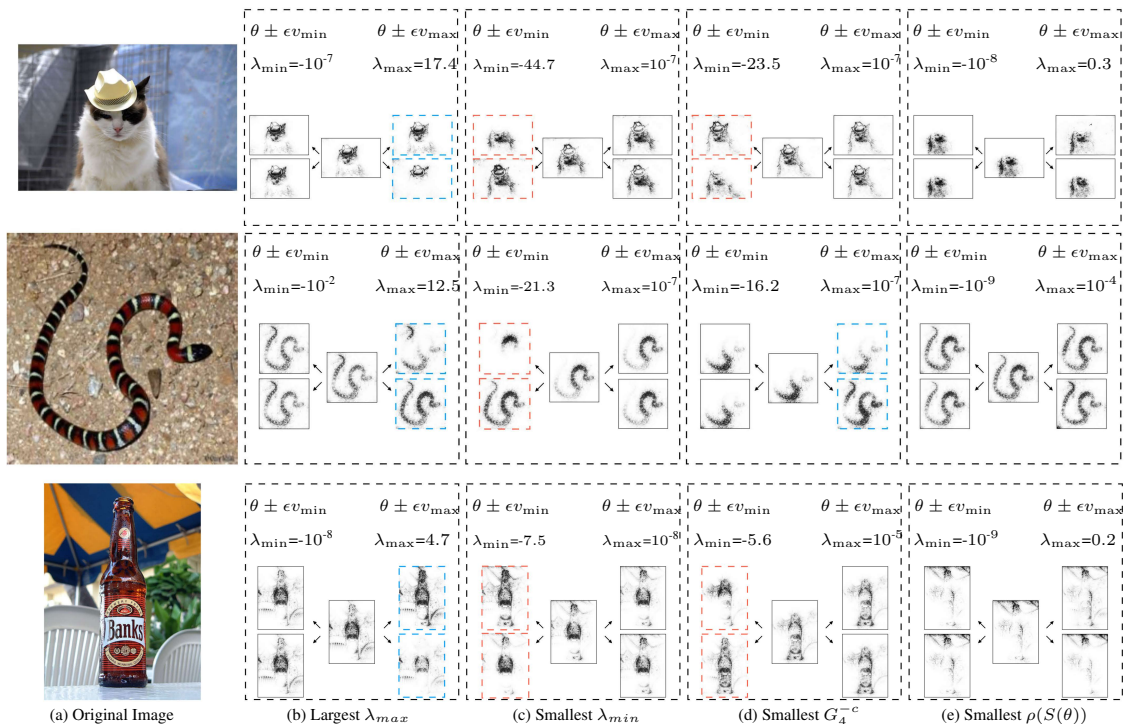

\subsection{Point Cloud Classification}
\label{app:dgcnn}
 DGCNN \cite{wang2019dynamic} has two types of layers, the EdgeConv layers and the fully-connected layers. 
 We keep the fully-connected layers  unchanged, and 
 apply S3D to split the filters in the EdgeConv layers which contain {expensive} dynamic KNN operations. The smaller number of filters can substantially speed up the KNN operations, leading a faster forward speed in the EdgeConv layers.
 
 The standard DGCNN has $(64, 64, 128, 256)$ channels in its 4  EdgeConv layers. 
 We initialize our splitting 
 process starting from a small DGCNN with 16 channels in each EdgeConv layer. 
 In the parametric update phase, we follow the setting of \citet{wang2019dynamic} and train the network with a batch size of 32 for 250 epochs. 
 We use the stochastic gradient descent with an initial learning rate of $0.1$, weight decay $10^{-4}$, and momentum $0.9$. 
 We use the cosine decay scheduler to decrease the learning rate. In each splitting phase, 
 we ue the gradient-based eigen-approximation following \citet{wang2019energy} and 
 increase the total number of filters by $40\%$. 
 {The splitting step size $\epsilon$ is $0.01$. }

\section{Algorithm in Practice}

\subsection{Selecting Splittings with Knapsack}
\label{sec:multi_case}
The different splitting schemes 
provide a trade-off between network size and loss descent.  
Although we mainly use binary splittings $(m=2)$ in the experiment, there can be cases when it is better to use different splitting schemes (i.e., different $m$) for different neurons to 
achieve the best possible improvement under a constraint of the growth of total network size. 
We can frame the optimal choice of the splitting schemes of different neurons into a simple knapsack problem. 

Specifically, assume we have a network with $N$ different neurons. Let $S_\ell$ be the splitting matrix of the $\ell$-th neuron and $G_{m,\ell}^{-c}$ the loss decrease obtained by splitting the $\ell$-th neuron into $m$ copies. We want to decide an optimal splitting size $m_\ell \in\{1,2,3,4\}$ (here $m_\ell=1$ means the neuron is not split) of the $\ell$-th neuron for each $\ell\in[N]$, 
subject to a predefined constraint of the total number of neurons after splitting ($m_{total}$). The optimum $\{m_\ell\}_{\ell=1}^{N}$ solve the following constrained optimization:  
\begin{align*} 
\min_{\{m_\ell\}}\sum_{\ell=1}^N G_{m_\ell, \ell}^{-c}, ~~s.t.~~ \sum_{\ell=1}^N m_{\ell} \leq m_{total},~~ m_\ell \in \{1,2,3,4\}. \\
\end{align*}
This is an instance of knapsack problem. In practice, we can solve it using convex relaxation. 
Specifically, let $p_{m,\ell}$ be a probability of splitting the $\ell$-th neuron into $m$ copies, we have 
\begin{align*} 
\min_{p} \sum_{\ell=1}^N \sum_{m=1}^4 p_{m, \ell} G_{m, \ell}^{-c}, ~~~~~s.t.~~~~~ & \sum_{\ell=1}^N \sum_{m=1}^4 m \times p_{m,\ell} \leq m_{total}  \\
& \sum_{m=1}^4 p_{m,\ell}=1, ~~~~~ \forall \ell\in[N]  \\
& p_{m,\ell}\ge 0 ~~~~~ \forall \ell\in[N],  ~~ m\in [4].
\end{align*}
In practice, we can solve the above problem using linear programming. 


\subsection{Result on CIFAR-100 using Knapsack}
We also test the result on CIFAR-100 using MobileNetV1, in which we decide the splitting schemes by formulating it as the knapsack problem. In this scenario, we consider all the situations we described in Section \ref{sec:method}. We show the result using the same setting as that in Table  \ref{table:cifar100}. As shown in Table \ref{table:cifar100-app}, the S3D with knapsack gives comparable accuracy as the Binary splitting, with slightly lower Flops.

\begin{table}[bthp]

\begin{center}
\setlength{\tabcolsep}{3pt}
\scalebox{0.8}{
\begin{tabular}{l|ccc}
\hline
Method & Accuracy & \# Param (M) & \# Flops (M) \\
\hline
S2D-5 & 69.69& 0.31 & 79.94\\
S3D-5 (m = 2) & \textbf{70.19}& \textbf{0.30} & 73.69\\
S3D-5 (knapsack) & 70.01& \textbf{0.30} & 72.19\\
\hline

\end{tabular}
}
\end{center}
\caption{Comparison of different methods near the full-size accuracy. 
S2D-5 denotes running splitting steepest descent (S2D) for 5 steps.
S3D-5 ($m = 2$) 
and S3D-5 (knapsack) 
represent running  S3D  for 5 steps use binary splitting and knapsack, respectively.}\label{table:cifar100-app}
\end{table}
\subsection{Fast Implementation
via Rayleigh Quotient
}  
\label{appendix:Fast_RQ}
\citet{wang2019energy} developed a 
Rayleigh-quotient gradient descent method for fast calculation of the minimum eigenvalues and eigenvectors of the splitting matrices. 
We extend it to calculate both minimum eigenvalues and maximum eigenvalues. 

Define the Rayleigh quotient \cite{parlett1998} of a matrix $S$ is 
$$
\mathcal R_S(v) := \frac{ v^\top S v }{v^\top v}.
$$
The maximum and minimum of 
$\mathcal R_S(v)$ equal the maximum and minimum eigenvalues, respectively, that is, 
\begin{align*} 
  &\lambda_{\min} 
= \min_{v}
\mathcal R_S(v), 
&&
v_{\min} \propto \argmin_{v}  \mathcal R_S(v) \\
&\lambda_{\max} 
= \max_{v}
\mathcal R_S(v)
, 
&&
v_{\min} \propto \argmax_{v}  \mathcal R_S(v).
\label{eq:rayley_}
\end{align*}
Therefore, we can apply gradient descent and gradient ascent on Rayleigh quotient
to approximate the minimum  and maximum eigenvalues. 
In practice, we use the automatic differentiation trick proposed by 
\citet{wang2019energy} to simultaneously calculate the gradient of 
Rayleigh quotient of all the neurons without using for loops,
by only evaluating the matrix-vector products $Sv$, without expanding the full splitting matrices.   


\end{document}